\documentclass{article}
\usepackage{amssymb}
\usepackage{amsmath,amsfonts,amsthm,amssymb}
\usepackage{multicol,lipsum}
 \usepackage{relsize}

\usepackage{soul,color}
\usepackage{float}
\usepackage{caption}
\usepackage{fullpage}

\usepackage{graphicx,float,wrapfig}
\PassOptionsToPackage{numbers}{natbib}
\pdfoutput=1
\def\showkeys{0}
\def\showdraftbox{1}
\def\showcolorlinks{1}
\def\usemicrotype{1}
\def\showfixme{0}

\usepackage{natbib}
\usepackage{bbm,amsmath,amsthm,amsfonts, amsopn, amssymb}
\usepackage{color}
\usepackage{graphicx, subcaption, epstopdf}
\usepackage{hyperref}
\usepackage{mathtools}
\usepackage[shortlabels]{enumitem}

\newtheorem{innercustomthm}{Claim}
\newenvironment{customthm}[1]
{\renewcommand\theinnercustomthm{#1}\innercustomthm}
{\endinnercustomthm}

\newtheorem{theorem}{Theorem}[section]

\newtheorem{assumption}{Assumption}

\newtheorem{definition}[theorem]{Definition}
\newtheorem{lemma}[theorem]{Lemma}
\newtheorem{proposition}[theorem]{Proposition}

\newcommand{\KKT}{optimality~}

\newcommand{\colspan}{\textup{col-span}}
\newcommand{\smin}{\sigma_{\min}}
\newcommand{\smax}{\sigma_{\max}}
\newcommand{\smaxz}{\smax(Z)}
\newcommand{\sminz}{\smin(Z)}

\newcommand{\BALD}{\begin{aligned}}
\newcommand{\EALD}{\end{aligned}}
\newcommand{\BALDS}{\begin{aligned*}}
\newcommand{\EALDS}{\end{aligned*}}
\newcommand{\BCAS}{\begin{cases}}
\newcommand{\ECAS}{\end{cases}}
\newcommand{\BEAS}{\begin{eqnarray*}}
\newcommand{\EEAS}{\end{eqnarray*}}
\newcommand{\BEQ}{\begin{equation}}
\newcommand{\EEQ}{\end{equation}}
\newcommand{\BIT}{\begin{itemize}}
\newcommand{\EIT}{\end{itemize}}
\newcommand{\BMAT}{\begin{bmatrix}}
\newcommand{\EMAT}{\end{bmatrix}}
\newcommand{\BNUM}{\begin{enumerate}}
\newcommand{\ENUM}{\end{enumerate}}

\newcommand{\BA}{\begin{array}}
\newcommand{\EA}{\end{array}}

\newcommand{\diag}{\mathop{\mathbf{diag}}}

\DeclareMathOperator{\trace}{tr}

\newcommand{\inner}[1]{\langle #1\rangle }

\usepackage{etex}

\usepackage[l2tabu, orthodox]{nag}

\usepackage{xspace,enumerate}

\usepackage[dvipsnames]{xcolor}

\usepackage[full]{textcomp}

\usepackage[american]{babel}

\usepackage{mathtools}

\usepackage{amsthm}

\usepackage{algorithm}
\usepackage{algorithmic}

\newtheorem*{fact*}{Fact}

\newtheorem*{hypothesis*}{Hypothesis}

\theoremstyle{definition}

\theoremstyle{remark}
\newtheorem{claim}[theorem]{Claim}
\newtheorem*{claim*}{Claim}

\newtheorem*{remark*}{Remark}

\newtheorem*{observation*}{Observation}

\usepackage{textcomp} 
\usepackage[scr=rsfso]{mathalfa}
\usepackage{bm} 

\ifnum\showkeys=1
\usepackage[color]{showkeys}
\fi

\ifnum\showcolorlinks=0
\usepackage[
pagebackref,
colorlinks=false,
pdfborder={0 0 0}
]{hyperref}
\fi

\usepackage{prettyref}

\newcommand{\savehyperref}[2]{\texorpdfstring{\hyperref[#1]{#2}}{#2}}

\newrefformat{eq}{\savehyperref{#1}{\textup{(\ref*{#1})}}}
\newrefformat{eqn}{\savehyperref{#1}{\textup{(\ref*{#1})}}}
\newrefformat{lem}{\savehyperref{#1}{Lemma~\ref*{#1}}}
\newrefformat{def}{\savehyperref{#1}{Definition~\ref*{#1}}}
\newrefformat{thm}{\savehyperref{#1}{Theorem~\ref*{#1}}}
\newrefformat{cor}{\savehyperref{#1}{Corollary~\ref*{#1}}}
\newrefformat{cha}{\savehyperref{#1}{Chapter~\ref*{#1}}}
\newrefformat{sec}{\savehyperref{#1}{Section~\ref*{#1}}}
\newrefformat{app}{\savehyperref{#1}{Appendix~\ref*{#1}}}
\newrefformat{tab}{\savehyperref{#1}{Table~\ref*{#1}}}
\newrefformat{fig}{\savehyperref{#1}{Figure~\ref*{#1}}}
\newrefformat{hyp}{\savehyperref{#1}{Hypothesis~\ref*{#1}}}
\newrefformat{alg}{\savehyperref{#1}{Algorithm~\ref*{#1}}}
\newrefformat{rem}{\savehyperref{#1}{Remark~\ref*{#1}}}
\newrefformat{item}{\savehyperref{#1}{Item~\ref*{#1}}}
\newrefformat{step}{\savehyperref{#1}{step~\ref*{#1}}}
\newrefformat{conj}{\savehyperref{#1}{Conjecture~\ref*{#1}}}
\newrefformat{fact}{\savehyperref{#1}{Fact~\ref*{#1}}}
\newrefformat{prop}{\savehyperref{#1}{Proposition~\ref*{#1}}}
\newrefformat{prob}{\savehyperref{#1}{Problem~\ref*{#1}}}
\newrefformat{claim}{\savehyperref{#1}{Claim~\ref*{#1}}}
\newrefformat{relax}{\savehyperref{#1}{Relaxation~\ref*{#1}}}
\newrefformat{red}{\savehyperref{#1}{Reduction~\ref*{#1}}}
\newrefformat{part}{\savehyperref{#1}{Part~\ref*{#1}}}

\newcommand{\Sref}[1]{\hyperref[#1]{\S\ref*{#1}}}

\usepackage{nicefrac}

\ifnum\usemicrotype=1
\usepackage{microtype}
\fi

\definecolor{forestgreen(traditional)}{rgb}{0.0, 0.27, 0.13}
\definecolor{brightgreen}{rgb}{0.0, 0.87, 0.13}

\ifnum\showfixme=0

\fi

\usepackage{boxedminipage}

\newcommand{\Paren}[1]{\left(#1\right)}

\newcommand{\Set}[1]{\left\{#1\right\}}

\newcommand{\norm}[1]{\lVert#1\rVert}
\newcommand{\Norm}[1]{\left\lVert#1\right\rVert}

\newcommand{\Esymb}{\mathbb{E}}
\newcommand{\Psymb}{\mathbb{P}}
\newcommand{\Vsymb}{\mathbb{V}}

\DeclareMathOperator*{\E}{\Esymb}
\DeclareMathOperator*{\Var}{\Vsymb}
\DeclareMathOperator*{\ProbOp}{\Psymb}

\renewcommand{\Pr}{\ProbOp}

\newcommand{\textparen}[1]{\text{(#1)}}

\ifx\because\undefined
\newcommand{\because}[1]{\textparen{because #1}}
\else
\renewcommand{\because}[1]{\textparen{because #1}}
\fi

\newcommand{\mper}{\,.}
\newcommand{\mcom}{\,,}

\newcommand\bdot\bullet

\ifx\mathds\undefined 
\DeclareMathOperator{\Ind}{\mathbb{I}}
\else
\DeclareMathOperator{\Ind}{\mathds 1}}
\fi

\DeclareMathOperator{\poly}{poly}

\DeclareMathOperator{\argmax}{argmax}

\DeclareMathOperator{\sign}{sign}

\newcommand{\R}{\mathbb R}

\renewcommand{\le}{\leqslant}

\renewcommand{\ge}{\geqslant}

\ifnum\showdraftbox=1

\else

\fi

\let\epsilon=\varepsilon

\numberwithin{equation}{section}

\newcommand\MYcurrentlabel{xxx}

\newcommand{\MYstore}[2]{%
  \global\expandafter \def \csname MYMEMORY #1 \endcsname{#2}%
}

\newcommand{\MYload}[1]{%
  \csname MYMEMORY #1 \endcsname%
}

\newcommand{\MYnewlabel}[1]{%
  \renewcommand\MYcurrentlabel{#1}%
  \MYoldlabel{#1}%
}

\newcommand{\MYdummylabel}[1]{}

\newcommand{\torestate}[1]{%
  \let\MYoldlabel\label%
  \let\label\MYnewlabel%
  #1%
  \MYstore{\MYcurrentlabel}{#1}%
  \let\label\MYoldlabel%
}

\newcommand{\restatetheorem}[1]{%
  \let\MYoldlabel\label
  \let\label\MYdummylabel
  \begin{theorem*}[Restatement of \prettyref{#1}]
    \MYload{#1}
  \end{theorem*}
  \let\label\MYoldlabel
}

\newcommand{\restatelemma}[1]{%
  \let\MYoldlabel\label
  \let\label\MYdummylabel
  \begin{lemma*}[Restatement of \prettyref{#1}]
    \MYload{#1}
  \end{lemma*}
  \let\label\MYoldlabel
}

\newcommand{\restateprop}[1]{%
  \let\MYoldlabel\label
  \let\label\MYdummylabel
  \begin{proposition*}[Restatement of \prettyref{#1}]
    \MYload{#1}
  \end{proposition*}
  \let\label\MYoldlabel
}

\newcommand{\restatefact}[1]{%
  \let\MYoldlabel\label
  \let\label\MYdummylabel
  \begin{fact*}[Restatement of \prettyref{#1}]
    \MYload{#1}
  \end{fact*}
  \let\label\MYoldlabel
}

\newcommand{\restate}[1]{%
  \let\MYoldlabel\label
  \let\label\MYdummylabel
  \MYload{#1}
  \let\label\MYoldlabel
}

\newcommand{\eps}{\epsilon}

\let\origparagraph\paragraph
\renewcommand{\paragraph}[1]{\origparagraph{#1.}}

\allowdisplaybreaks

\usepackage{paralist}

\usepackage{comment}

\usepackage{times}
\newcommand{\Po}{P_{\Omega}}
\def\bar{\overline}

\begin{document}
\title{Matrix Completion has No Spurious Local Minimum}
\author{Rong Ge\thanks{Duke University, rongge@cs.duke.edu. }\and Jason D. Lee\thanks{University of Southern California, jasonlee@marshall.usc.edu. }\and Tengyu Ma\thanks{Princeton University, tengyu@cs.princeton.edu. Supported in part by Simons Award in Theoretical Computer Science and IBM PhD Fellowship.}}
\maketitle

\begin{abstract}
Matrix completion is a basic machine learning problem that has wide applications, especially in collaborative filtering and recommender systems. Simple non-convex optimization algorithms are popular and effective in practice. Despite recent progress in proving various non-convex algorithms converge from a good initial point, it remains unclear why random or arbitrary initialization suffices in practice. 
We prove that the commonly used non-convex objective function for \textit{positive semidefinite} matrix completion has no spurious local minima \--- all local minima must also be global. Therefore, many popular optimization algorithms such as (stochastic) gradient descent can provably solve positive semidefinite matrix completion with \textit{arbitrary} initialization in polynomial time. The result can be generalized to the setting when the observed entries contain noise. We believe that our main proof strategy can be useful for understanding geometric properties of other statistical problems involving partial or noisy observations. 

\end{abstract}

\section{Introduction}

Matrix completion is the problem of recovering a low rank matrix from partially observed entries. It has been widely used in collaborative filtering and recommender systems~\cite{koren2009bellkor,rennie2005fast}, dimension reduction \cite{candes2011robust} and multi-class learning \cite{amit2007uncovering}. 
There has been extensive work on designing efficient algorithms for matrix completion with guarantees. One earlier line of results (see~\cite{recht2011simpler,candes2010power,candes2009exact} and the references therein) rely on convex relaxations.  These algorithms achieve strong statistical guarantees, 
but are quite computationally expensive in practice.

More recently,
there has been growing interest in analyzing non-convex algorithms for matrix completion \cite{keshavan2010matrix,keshavan2010matrixnoisy,jain2013low,hardt2014understanding, hardt2014fast,sun2015guaranteed, zhao2015nonconvex,chen2015fast,DBLP:conf/icml/SaRO15,chen2015fast}. Let $M \in \R^{d\times d}$ be the target matrix with rank $r \ll d$ that we aim to recover, and let $\Omega = \{(i,j): M_{i,j}\mbox{ is observed}\}$ be the set of observed entries. These methods are instantiations of optimization algorithms applied to the objective\footnote{In this paper, we focus on the symmetric case when the true $M$ has a symmetric decomposition $M = ZZ^T$. Some of previous papers work on the asymmetric case when $M = ZW^T$, which is harder than the symmetric case. }, 
\begin{equation} \label{eq:obj}
f(X) = \frac{1}{2} \sum_{(i,j)\in \Omega} \left[M_{i,j} - (XX^{\top})_{i,j}\right]^2,
\end{equation}
These algorithms are much faster than the convex relaxation algorithms, which is crucial for their empirical success in large-scale collaborative filtering applications \cite{koren2009bellkor}.

Most of the theoretical analysis of the nonconvex procedures require careful initialization schemes: the initial point should already be close to optimum\footnote{The work of De Sa et al.~\cite{DBLP:conf/icml/SaRO15} is an exception, which gives an algorithm that uses fresh samples at every iteration to solve matrix completion (and other matrix problems) approximately.}.  In fact, Sun and Luo~\cite{sun2015guaranteed} showed that after this initialization the problem is effectively strongly-convex, hence many different optimization procedures can be analyzed by standard techniques from convex optimization. 

However, in practice people typically use a random initialization, which still leads to robust and fast convergence. Why can these practical algorithms find the optimal solution in spite of the non-convexity?
In this work we investigate this question and show that the matrix completion objective has {\em no spurious} local minima. More precisely, we show that any local minimum $X$ of objective function $f(\cdot)$ is also a global minimum with $f(X) = 0$, and recovers the correct low rank matrix $M$.

Our characterization of the structure in the objective function implies that (stochastic) gradient descent from arbitrary starting point converge to a global minimum. This is because gradient descent converges to a local minimum~\cite{ge2015escaping,lee2016gradient}, and every local minimum is also a global minimum.

\subsection{Main results}

Assume the target matrix $M$ is symmetric and each entry of $M$ is observed with probability $p$ independently \footnote{The entries $(i,j)$ and $(j,i)$ are the same. With probability $p$ we observe both entries and otherwise we observe neither.}. We assume $M = ZZ^{\top}$ for some matrix $Z \in \R^{d\times r}$. 

There are two known issues with matrix completion. First, the choice of $Z$ is not unique since $M = (ZR)(ZR)^{\top}$ for any orthonormal matrix $Z$. Our goal is to find one of these equivalent solutions.

Another issue is that matrix completion is impossible when $M$ is ``aligned'' with standard basis. For example, when $M$ is the identity matrix in its first $r\times r$ block, we will very likely be observing only 0 entries. To address this issue, we make the following standard assumption:

\begin{assumption} \label{assump:incoherence}
	For any row $Z_i$ of $Z$, we have
	$$
	\|Z_i\|\le \mu/\sqrt{d}\cdot \|Z\|_F.
	$$
	Moreover, $Z$ has a bounded condition number $\smax(Z)/\smin(Z) = \kappa$.
\end{assumption}
Throughout this paper we think of $\mu$ and $\kappa$ as small constants, and the sample complexity depends polynomially on these two parameters. Also note that this assumption is independent of the choice of $Z$: all $Z$ such that $ZZ^T = M$ have the same row norms and Frobenius norm.

This assumption is similar to the ``incoherence'' assumption \cite{candes2009exact}. Our assumption is the same as the one used in analyzing non-convex algorithms \cite{keshavan2010matrix,keshavan2010matrixnoisy,sun2015guaranteed}.

We enforce $X$ to also satisfy this assumption by a regularizer
\begin{equation}
f(X) = \frac{1}{2} \sum_{(i,j)\in \Omega} \left[M_{i,j} - (XX^{\top})_{i,j}\right]^2 + R(X),
\label{eq:reg-obj}
\end{equation}
where $R(X)$ is a function that penalizes $X$ when one of its rows is too large. See Section~\ref{sec:rank1} and Section~\ref{sec:rankr} for the precise definition.
Our main result shows that in this setting, the regularized objective function has no spurious local minimum:

\begin{theorem}\label{thm:main} [Informal] All local minimum of the regularized objective (\ref{eq:reg-obj}) satisfy $XX^T = ZZ^T = M$ when $p \ge \mbox{poly}(\kappa,r,\mu,\log d)/d$. 
\end{theorem}

Combined with the results in \cite{ge2015escaping, lee2016gradient} (see Theorem~\ref{thm:optimizer}) \footnote{Theorem~\ref{thm:main}, as state informally above, doesn't guarantee that $f$ satisfies the condition in Theorem~\ref{thm:optimizer}. See its technical version, Theorem~\ref{thm:main-rank-r}, which shows that the condition of Theorem~\ref{thm:optimizer} is satisfied.} , we have, 

\begin{theorem} [Informal] With high probability, stochastic gradient descent on the regularized objective (\ref{eq:reg-obj}) will converge to a solution $X$ such that $XX^T = ZZ^T = M$ in polynomial time from {\em any} starting point. Gradient descent will converge to such a point with probability 1 from a random starting point.
\end{theorem}

Our results are also robust to noise. Even if each entry is corrupted with Gaussian noise of standard deviation $\mu^2\|Z\|_F^2/d$ (comparable to the magnitude of the entry itself!), we can still guarantee that all the local minima satisfy $\|XX^T - ZZ^T\|_F \le \epsilon$ when $p$ is large enough. See the discussion in Appendix~\ref{sec:noise} for results on noisy matrix completion.

Our main technique is to show that every point that satisfies the first and second order necessary conditions for optimality must be a desired solution. To achieve this we use new ideas to analyze the effect of the regularizer and show how it is useful in modifying the first and second order conditions to exclude any spurious local minimum.

\subsection{Related Work}\label{subsec:related_work}

\paragraph{Matrix Completion} The earlier theoretical works on matrix completion analyzed the nuclear norm minimization \cite{srebro2005rank,recht2011simpler,candes2010power,candes2009exact, negahban2012restricted}. This line of work has the cleanest and strongest theoretical guarantees; \cite{candes2010power,recht2011simpler} showed that if $|\Omega| \gtrsim dr  \mu^2  \log^2 d$ the nuclear norm convex relaxation recovers the exact underlying low rank matrix. The solution can be computed via the solving a convex program in polynomial time. However the primary disadvantage of nuclear norm methods is their computational and memory requirements --- the fastest known provable algorithms require $O(d ^2)$ memory and thus at least $O(d^2)$ running time, which could be both prohibitive for moderate to large values of $d$. Many algorithms have been proposed to improve the runtime (either theoretically or empirically) (see, for examples, ~\cite{srebro2004maximum,mazumder2010spectral,hastie2014matrix}, and the reference therein). 
Burer and Monteiro~\cite{burer2003nonlinear} proposed factorizing the optimization variable $\widehat M = XX^T$, and optimizing over $X \in \mathbb{R}^{d \times r}$  instead of $\widehat M \in \mathbb{R}^{d \times d}$. This approach only requires $O(dr)$ memory, and a single gradient iteration takes time $O(|\Omega|)$, so has much lower memory requirement and computational complexity than the nuclear norm relaxation. On the other hand, the factorization causes the optimization problem to be non-convex in $X$, which leads to theoretical difficulties in analyzing algorithms. 
Keshavan et al.~\cite{keshavan2010matrix,keshavan2010matrixnoisy} showed that well-initialized gradient descent recovers $M$. The works
\cite{ hardt2014fast,hardt2014understanding,jain2013low,chen2015fast} showed that well-initialized alternating least squares, block coordinate descent, and gradient descent converges $M$. Jain and Netrapalli~\cite{jain2015fast} showed a fast algorithm by iteratively doing gradient descent in the relaxed space and projecting to the set of low-rank matrices. 
The work~\cite{DBLP:conf/icml/SaRO15} analyzes stochastic gradient descent with fresh samples at each iteration from random initialization and shows that it approximately converge to the optimal solution.~\cite{sun2015guaranteed, zhao2015nonconvex,zheng2016convergence,tu2015low} provided a more unified analysis by showing that with careful initialization many algorithms, including gradient descent and alternating least squares, succeed. \cite{sun2015guaranteed, zheng2016convergence} accomplished this by showing an analog of strong convexity in the neighborhood of the solution $M$. 

\paragraph{Non-convex Optimization} Recently, a line of work analyzes non-convex optimization by separating the problem into two aspects: the geometric aspect which shows the function has no spurious local minimum and the algorithmic aspect which designs efficient algorithms can converge to local minimum that satisfy first and (relaxed versions) of second order necessary conditions.

Our result is the first that explains the geometry of the matrix completion objective. Similar geometric results are only known for a few problems: SVD/PCA phase retrieval/synchronization, orthogonal tensor decomposition, dictionary learning \cite{baldi1989neural,srebro2003weighted, ge2015escaping,sun2015nonconvex, bandeira2016low}. The matrix completion objective requires different tools due to the sampling of the observed entries, as well as carefully managing the regularizer to restrict the geometry. Parallel to our work Bhojanapalli et al.\cite{bhojanapalli2016personal} showed similar results for matrix sensing, which is closely related to matrix completion. Loh and Wainwright~\cite{DBLP:journals/jmlr/LohW15} showed that for many statistical settings that involve missing/noisy data and non-convex regularizers, any stationary point of the non-convex objective is close to global optima; furthermore, there is a unique stationary point that is the global minimum under stronger assumptions \cite{loh2014support}.

On the algorithmic side, it is known that second order algorithms like cubic regularization \cite{nesterov2006cubic} and trust-region \cite{sun2015nonconvex} algorithms converge to local minima that approximately satisfy first and second order conditions. Gradient descent is also known to converge to local minima \cite{ lee2016gradient} from a random starting point. Stochastic gradient descent can converge to a local minimum in polynomial time from any starting point \cite{pemantle1990nonconvergence,ge2015escaping}. All of these results can be applied to our setting, implying various heuristics used in practice are guaranteed to solve matrix completion.

\section{Preliminaries}

\subsection{Notations}

For $\Omega\subset [d]\times[d]$, let $P_{\Omega}$ be the linear operator that maps a matrix $A$ to $P_{\Omega}(A)$, where $P_{\Omega}(A)$ has the same values as $A$ on $\Omega$, and $0$ outside of $\Omega$.

We will use the following matrix norms: $\|\cdot\|_F$ the frobenius norm, $\|\cdot\|$ spectral norm, $|A|_{\infty}$ elementwise infinity norm, and $|A|_{p \rightarrow q }= \max_{\|x\|_p =1} \|A\|_q$. We use the shorthand $\|A\|_{\Omega} = \|P_{\Omega}A\|_F$. The trace inner product of two matrices is $\inner{A,B}=\trace(A^{\top}B)$, and $\sigma_{\min} (X)$, $ \sigma_{\max}(X)$ are the smallest and largest singular values of $X$. We also use $X_i$ to denote the $i$-th row of a matrix $X$.

\subsection{Necessary conditions for Optimality}

Given an objective function $f(x):\R^n\to \R$, we use $\nabla f(x)$ to denote the gradient of the function, and $\nabla^2 f(x)$ to denote the Hessian of the function ($\nabla^2 f(x)$ is an $n\times n$ matrix where $[\nabla^2 f(x)]_{i,j} = \frac{\partial^2}{\partial x_i \partial x_j} f(x)$). It is well known that local minima of the function $f(x)$ must satisfy some necessary conditions:

\begin{definition} A point $x$ satisfies the first order necessary condition for optimality (later abbreviated as first order \KKT condition) if $\nabla f(x) = 0$. A point $x$ satisfies the second order necessary condition for optimality (later abbreviated as second order \KKT condition)if $\nabla^2 f(x) \succeq 0$.
\end{definition}

These conditions are necessary for a local minimum because otherwise it is easy to find a direction where the function value decreases.
We will also consider a relaxed second order necessary condition, where we only require the smallest eigenvalue of the Hessian $\nabla^2 f(x)$ to be not very negative:

\begin{definition} For $\tau \ge 0$, a point $x$ satisfies the $\tau$-relaxed second order \KKT condition, if $\nabla^2 f(x) \succeq -\tau \cdot I$.
\end{definition}

This relaxation to the second order condition makes the conditions more robust, and allows for efficient algorithms.

\begin{theorem} \cite{nesterov2006cubic, sun2015nonconvex, ge2015escaping} \label{thm:optimizer}
	Let $f$ be twice-differentiable function form $\R^d$ to $\R$. Suppose there exist $\epsilon_0, \tau_0 > 0$ and a universal constant $c> 0$ such that if a point  $x$ satisfies $\norm{\nabla f(x)}\le \epsilon\le \epsilon_0$ and $\nabla^2 f(x)\succeq -\tau_0\cdot I$, then $x$ is $\epsilon^c$-close to a global minimum of $f$. Then, many optimization algorithms including cubic regularization, trust-region, and stochastic gradient descent,  can find a global minimum of $f$ up to $\delta$ error in $\ell_2$ norm in domain in time $\poly(1/\delta,1/\tau_0, d)$.
\end{theorem}

\section{Proof Strategy: ``simple'' proofs are more generalizable}
\label{sec:intuition}

In this section, we demonstrate the key ideas behind our analysis using the rank $r=1$ case. In particular, we first give a ``simple'' proof for the fully observed case. Then we show this simple proof can be easily generalized to the random observation case. We believe that this proof strategy is applicable to other statistical problems involving partial/noisy observations. The proof sketches in this section are only meant to be illustrative and may not be fully rigorous in various places. We refer the readers to Section~\ref{sec:rank1} and Section~\ref{sec:rankr} for the complete proofs. 

In the rank $r=1$ case,  we assume $M= zz^{\top}$, where $\|z\|= 1$, and $\|z\|_{\infty}\le \frac{\mu}{\sqrt{d}}$. Let $\epsilon \ll 1$ be the target accuracy that we aim to achieve in this section and let  $p = \mbox{poly}(\mu,\log d)/(d\epsilon)$. 

For simplicity, we focus on the following domain $\mathcal{B}$ of incoherent vectors where the regularizer $R(x)$ vanishes,  \begin{align}
\mathcal{B} & = \Set{x: \|x\|_{\infty}  < \frac{2\mu}{\sqrt{d}}} \mper\label{eqn:x_incoherent_0}
\end{align}

Inside this domain $\mathcal{B}$, we can restrict our attention to the objective function without the regularizer, defined as, 
\begin{align}
\tilde{g}(x) = \frac12\cdot \norm{\Po( M-xx^{\top})}_F ^2 \mper
\end{align}

The global minima of $\tilde{g}(\cdot)$ are $z$ and $-z$ with function value 0. Our goal of this section is to (informally) prove that all the local minima of $\tilde{g}(\cdot)$ are $O(\sqrt{\epsilon})$-close to $\pm z$. In later section we will formally prove that the only local minima are $\pm z$.

\begin{lemma}[Partial observation case, informally stated]\label{lem:partial}
	Under the setting of this section, in the 
	domain $\mathcal{B}$, all local mimina of the function $\tilde{g}(\cdot)$ are  $O(\sqrt{\epsilon})$-close to  $\pm z$. 
\end{lemma}

It turns out to be insightful to consider the full observation case when $\Omega = [d]\times [d]$. The corresponding objective is 
\begin{align}
g(x) = \frac12\cdot \norm{ M-xx^{\top}}_F ^2 \mper
\end{align}

Observe that $\tilde{g}(x)$ is a sampled version of the  $g(x)$, and therefore we expect that they share the same geometric properties. In particular, if $g(x)$ does not have spurious local minima then neither does $\tilde{g}(x)$. 
\begin{lemma}[Full observation case, informally stated]\label{lem:full}
	Under the setting of this section, in the domain $\mathcal{B}$, the function $g(\cdot)$ has only two local minima $\{\pm z\}$ . 
\end{lemma}

Before introducing the ``simple'' proof, let us first look at a delicate proof that does not generalize well.

\paragraph{Difficult to Generalize Proof of Lemma~\ref{lem:full}} We compute the gradient and Hessian of $g(x)$, 
\begin{align}
\nabla g(x) = Mx - \|x\|^2 x, \textup{     and      } 
\nabla^2 g(x) =  2xx^{\top} -  M + \|x\|^2 \cdot I\mper
\end{align}

Therefore, a critical point $x$ satisfies $\nabla g(x) = Mx - \|x\|^2 x = 0$, and thus it must be an eigenvector of $M$ and $\|x\|^2$ is the corresponding eigenvalue. Next, we prove that the hessian is only positive definite at the top eigenvector . Let $x$ be an eigenvector with eigenvalue $\lambda = \|x\|^2$, and $\lambda$ is strictly less than the top eigenvalue $\lambda^*$.  Let $z$ be the top eigenvector.  We have that $\inner{z,\nabla^2 g(x)z} = -\inner{z,Mz} + \|x\|^2 = -\lambda^* + \lambda  < 0$, which shows that $x$ is not a local minimum. Thus only $z$ can be a local minimizer, and it is easily verified that $\nabla^2 g(z)$ is indeed positive definite.

The difficulty of generalizing the proof above to the partial observation case is that it uses the {\em properties of eigenvectors} heavily. Suppose we want to imitate the proof above for the partial observation case, the first difficulty is how to solve the equation $\tilde{g}(x) = P_{\Omega}(M-xx^{\top})x= 0$. Moreover, even if we could have a reasonable approximation for the critical points (the solution of $\nabla \tilde{g}(x) = 0$), it would be difficult to examine the Hessian of these critical points  without having the orthogonality of the eigenvectors. 

\paragraph{``Simple'' and Generalizable proof}

The lessons from the subsection above suggest us find an alternative proof for the full observation case which is generalizable. 
The alternative proof will be simple in the sense that it doesn't use the notion of eigenvectors and eigenvalues.
\vspace{.05in}
\noindent Concretely, the key observation behind most of the analysis in this paper is the following, 

\vspace{.05in}
 \textit{Proofs that consist of inequalities that are linear in $\mathbf{1}_{\Omega}$ are often easily generalizable to partial observation case.} 
\vspace{.05in}

\noindent 
Here statements that are linear in $\mathbf{1}_{\Omega}$ mean the statements of the form $\sum_{ij} 1_{(i,j) \in \Omega} T_{ij} \le a$. We will call these kinds of proofs ``simple'' proofs in this section. 
Roughly speaking, the observation follows from the law of large numbers --- Suppose $T_{ij}, (i,j)\in [d]\times [d]$ is a sequence of bounded real numbers, then the sampled sum $\sum_{(i,j)\in \Omega} T_{ij} = \sum_{i,j} \mathbf{1}_{(i,j)\in \Omega}T_{ij}$ is an accurate estimate of the sum $p\sum_{i,j}T_{ij}$, when the sampling probability $p$ is relatively large. Then, the mathematical implications of $p\sum T_{ij} \le a$ are expected to be similar to the implications of $\sum_{(i,j)\in \Omega} T_{ij} \le a$, up to some small error introduced by the approximation.
To make this concrete, we give below informal proofs for Lemma~\ref{lem:full} and Lemma~\ref{lem:partial} that only consists of statements that are linear in $\mathbf{1}_{\Omega}$. Readers will see that due to the linearity, the proof for the partial observation case (shown on the right column) is a direct generalization of the proof for the full observation case (shown on the left column) via concentration inequalities (which will be discussed more at the end of the section). \vspace{.1in}

\noindent
\begin{minipage}{\textwidth}
\begin{multicols}{2}
	
{\bf A ``simple'' proof for Lemma~\ref{lem:full}.}
\vspace{0.1in}
\begin{customthm}{1f}\label{claim:full1}
	Suppose $x\in \mathcal{B}$ satisfies $\nabla g(x) = 0$, then $\inner{x,z}^2 = \|x\|^4$.
\end{customthm}
\begin{proof} We have, 
\begin{align}
& \nabla g(x) = (zz^{\top}-xx^{\top})x= 0  \nonumber\\
\Rightarrow ~~& \inner{x,\nabla g(x)} = \inner{x,(zz^{\top}-xx^{\top})x} = 0 \quad\quad \quad \label{eqn:91}\\
\Rightarrow ~~& \inner{x,z}^2 =\|x\|^4 \nonumber
\end{align}
Intuitively, this proof says that the norm of a critical point $x$ is controlled by its correlation with $z$. {\color{white}Here at the lasa sampling version of the f}
{\color{white} aaaaaaaaaaaaaaaaaaaaaaaaaaaaaaaaaaaaaaaaaaaaaaaaaaaaaaaaaaaaaaaa}
\end{proof}

\paragraph{Generalization to Lemma~\ref{lem:partial}}
\begin{customthm}{1p}\label{claim:partial1}
 Suppose $x\in \mathcal{B}$ satisfies $\nabla \tilde{g}(x) = 0$, then $\inner{x,z}^2 = \|x\|^4 -\epsilon$.
\end{customthm}
\begin{proof} Imitating the proof on the left, we have
\begin{align}
& \nabla \tilde{g}(x) = P_{\Omega}(zz^{\top}-xx^{\top})x= 0  \nonumber\\
\Rightarrow ~~& \inner{x,\nabla \tilde{g}(x)} = \inner{x,P_{\Omega}(zz^{\top}-xx^{\top})x} = 0\quad\quad  \label{eqn:92}\\
\Rightarrow ~~& \inner{x,z}^2 \ge \|x\|^4 -\epsilon\nonumber\end{align}
The last step uses the fact that equation~\eqref{eqn:91} and~\eqref{eqn:92} are approximately equal up to scaling factor $p$ for any $x\in \mathcal{B}$, since ~\eqref{eqn:92} is a sampled version of~\eqref{eqn:91}. 
\end{proof}

\end{multicols}
\end{minipage}
\noindent
\begin{minipage}{\textwidth}
\begin{multicols}{2}
	\begin{customthm}{2f}\label{claim:full2}
		If $x\in \mathcal{B}$ has positive Hessian $\nabla^2 g(x)\succeq 0$, then $\|x\|^2 \ge 1/3$.
	\end{customthm}
	
	\begin{proof}
		By the assumption on $x$, we have that $\inner{z,\nabla^2 g(x)z} \ge 0$. Calculating the quadratic form of the Hessian (see Proposition~\ref{prop:kktrank1} for details), 
		\begin{align}
		& \inner{z,\nabla^2 g(x)z} \nonumber\\
		& = \norm{zx^{\top}+xz^{\top}}^2 - 2z^{\top}(zz^{\top}-xx^{\top})z\ge 0 {\color{white} aaaaaa}\\
		& \Rightarrow \norm{x}^2 + 2\inner{z,x}^2 \ge 1 \nonumber\\
		& \Rightarrow \norm{x}^2\ge 1/3\tag{since $\inner{z,x}^2\le \norm{x}^2$}
		\end{align} 
		{\color{white} aaaaaaaaaaaaaaaaaaaaaaaaaaaaaaaaaaaaaaaaaaaaaaaaaaaaaaaaaaaaaaaa}
	\end{proof}

	\begin{customthm}{2p}\label{claim:partial2}
		If $x\in \mathcal{B}$ has positive Hessian $\nabla^2 \tilde{g}(x)\succeq 0$, then $\|x\|^2 \ge 1/3-\epsilon$.
	\end{customthm}
	\begin{proof}
		Imitating the proof on the left, calculating the quadratic form over the Hessian at $z$ (see Proposition~\ref{prop:kktrank1}) , we  have {\color{white} aaaaaaaaaaaaaaaaaaaaaaaaaaaaaaaa}
		\begin{align}
		& \inner{z,\nabla^2 \tilde{g}(x)z} \nonumber\\
		& = \norm{P_{\Omega}(zx^{\top}+xz^{\top})}^2 - 2z^{\top}P_{\Omega}(zz^{\top}-xx^{\top})z\ge 0 \\
		& \Rightarrow \cdots\cdots \tag{same step as the left}\nonumber\\
		& \Rightarrow \norm{x}^2\ge 1/3-\epsilon\nonumber
		\end{align} 
	Here we use the fact that $\inner{z,\nabla^2 \tilde{g}(x)z}\approx p\inner{z,\nabla^2 g(x)z}$ for any $x\in \mathcal{B}$. 
	\end{proof}
	
\end{multicols}
\end{minipage}
\vspace{.1in}

With these two claims, we are ready to prove Lemma~\ref{lem:full} and~\ref{lem:partial} by using another step that is linear in $\mathbf{1}_{\Omega}$. 

\vspace{.1in}
\noindent
\begin{minipage}{\textwidth}
\begin{multicols}{2}
		\begin{proof}[Proof of Lemma~\ref{lem:full}]
			By Claim~\ref{claim:full1} and~\ref{claim:full2}, we have $x$ satisfies $\inner{x,z}^2\ge \|x\|^4\ge 1/9$.
		Moreover, we have that $\nabla g(x) =  0 $ implies 
			\begin{align}
		& \inner{z,\nabla g(x)} = \inner{z,(zz^{\top}-xx^{\top})x} = 0\label{eqn:93}\\
			\Rightarrow ~~& \inner{x,z} (1 - \|x\|^2 ) = 0\nonumber\\
			\Rightarrow ~~& \|x\|^2  = 1 \tag{by $\inner{x,z}^2\ge 1/9$} 
						\end{align}
			Then by Claim~\ref{claim:full1} again we obtain $\inner{x,z}^2 = 1$, and therefore $x = \pm z$. 	{\color{white} aaaaaaaaaaaaaaaaaaaaaaaaaaaaaaaaaaaaaaaaaaaaaaaaaaaaaaaaaaaaaaaa}	{\color{white} aaaaaaaaaaaaaaaaaaaaaaaaaaaaaaaaaaaaaaaaaaaaaaaaaaaaaaaaaaaaaaaa}	{\color{white} aaaaaaaaaaaaaaaaaaaaaaaaaaaaaaaaaaaaaaaaaaaaaaaaaaaaaaaaaaaaaaaaaaaaaa}
		\end{proof}
		
		\begin{proof}[Proof of Lemma~\ref{lem:partial}]
			By Claim~\ref{claim:partial1} and~\ref{claim:partial2}, we have $x$ satisfies $\inner{x,z}^2\ge \|x\|^4\ge 1/9-O(\epsilon)$.
			Moreover, we have that $\nabla \tilde{g}(x) =  0 $ implies 
			\begin{align}
			& \inner{z,\nabla \tilde{g}(x)} = \inner{z,P_{\Omega}(zz^{\top}-xx^{\top})x} = 0\label{eqn:94}\\
			\Rightarrow ~~& \cdots\cdots \tag{same step as the left} \\		\Rightarrow ~~& \|x\|^2  = 1 \pm O(\epsilon)  \tag{same step as the left}
			\end{align}
			Since ~\eqref{eqn:94} is the sampled version of equation~\eqref{eqn:93}, we expect they lead to the same conclusion up to some approximation. 
			Then by Claim~\ref{claim:partial1} again we obtain $\inner{x,z}^2 = 1\pm O(\epsilon)$, and therefore $x$ is $O(\sqrt{\epsilon})$-close to either of $\pm z$. 
		\end{proof}
	\end{multicols}
\end{minipage}

\paragraph{Subtleties regarding uniform convergence} In the proof sketches above, our key idea is to use concentration inequalities to link the full observation objective $g(x)$ with the partial observation counterpart. However, we require a uniform convergence result. For example, we need a statement like ``w.h.p over the choice of $\Omega$, equation~\eqref{eqn:91} and~\eqref{eqn:92} are similar to each other up to scaling". This type of statement is often only true for $x$ inside the incoherent ball $\mathcal{B}$. 
The fix to this is the regularizer. For non-incoherent $x$, we will use a different argument that uses the property of the regularizer. This is besides the main proof strategy of this section and will be discussed in subsequent sections.

\section{Warm-up: Rank-1 Case}\label{sec:rank1}

In this section, using the general proof strategy described in previous section, we provide a formal proof for the rank-1 case. In subsection~\ref{subsec:incoherent}, we formally work out the proof sketches of Section~\ref{sec:intuition}. In subsection~\ref{subsec:nonincoherent}, we prove that due to the effect of the regularizer, outside incoherent ball $\mathcal{B}$, the objective function doesn't have any local minimum. 

In the rank-1 case, the objective function simplifies to,  \begin{equation}
f(x) = \frac12 \norm{\Po( M-xx^{\top})}_F ^2 + \lambda R(x) \mper\label{eqn:rank-1-objective}
\end{equation}
Here we use the the regularization $R(x)$ 
\begin{equation}
R(x) = \sum_{i=1}^{d} h(x_i), \textup{ and } h(t) =  (|t|-\alpha)^4\Ind_{t \ge \alpha} \mper \nonumber
\end{equation}
The parameters $\lambda$ and $\alpha$ will be chosen later as in Theorem~\ref{thm:rank1_main}. We will choose $\alpha > 10\mu/\sqrt{d}$ so that $R(x) = 0$ for incoherent $x$, and thus it only penalizes coherent $x$. Moreover, we note $R(x)$ has Lipschitz second order derivative. \footnote{This is the main reason for us to choose $4$-th power instead of $2$-nd power. }

We first state the \KKT conditions, whose proof is deferred to Appendix~\ref{sec:rank1proof}. 

\begin{proposition} \label{prop:kktrank1}
The first order \KKT condition of objective~\eqref{eqn:rank-1-objective} is, 
\begin{align}
2\Po (M -xx^{\top})x &=\lambda \nabla R(x)\label{eq:kkt1}\mcom
\end{align}
and the second order \KKT condition requires:
\begin{align}
\forall v\in \R^d, ~\norm{\Po(vx^{\top} +xv^{\top})}_F ^2 + \lambda v^{\top}\nabla^2 R(x)v&\ge 2v^{\top} \Po (M-xx^{\top})v \mper
\label{eq:kkt2}
\end{align}
Moreover, The $\tau$-relaxed second order \KKT condition requires
\begin{align}
\forall v\in \R^d, ~\norm{\Po(vx^{\top} +xv^{\top})}_F ^2 + \lambda v^{\top}\nabla^2 R(x)v&\ge 2v^{\top} \Po (M-xx^{\top})v - \tau\|v\|^2 \mper
\label{eq:kkt2relaxed}
\end{align}
\end{proposition}
We give the precise version of Theorem~\ref{thm:main} for the rank-$1$ case. 

\begin{theorem} \label{thm:rank1_main}
		For $p\ge \frac{c\mu^6\log^{1.5}d}{d}$ where $c$ is a large enough absolute constant, set $\alpha = 10\mu\sqrt{1/d}$ and $\lambda \ge \mu^2 p/\alpha^2$.Then, with high probability over the randomness of $\Omega$,  the only points in $\R^d$ that satisfy both first and  second order \KKT  conditions (or $\tau$-relaxed \KKT conditions with $\tau < 0.1p$) are $z$ and $-z$. \end{theorem}

In the rest of this section, we will first prove that when $x$ is constrained to be incoherent (and hence the regularizer is 0 and concentration is straightforward) and satisfies the optimality conditions, then $x$ has to be $z$ or $-z$. Then we go on to explain how the regularizer helps us to change the geometry of those points that are far away from $z$ so that we can rule out them from being local minimum.
For simplicity, we will focus on the part that shows a local minimum $x$ must be close enough to $z$. 
\begin{lemma} \label{thm:warmup-close}
In the setting of Theorem~\ref{thm:rank1_main}, suppose $x$ satisfies the first-order and second-order \KKT condition~\eqref{eq:kkt1} and~\eqref{eq:kkt2}. Then when $p$ is defined as in Theorem~\ref{thm:rank1_main},  
	\begin{equation}
	\Norm{xx^{\top} - zz^{\top}}_F^2\le O(\epsilon) \mper \nonumber
		\end{equation}
	where $\epsilon = \mu^3(pd)^{-1/2}$. 
\end{lemma}

This turns out to be the main challenge. Once we proved $x$ is close, we can apply the result of Sun and Luo \cite{sun2015guaranteed} (see Lemma~\ref{lem:exact}), and obtain Theorem~\ref{thm:rank1_main}.

\subsection{Handling incoherent $x$}\label{subsec:incoherent}

To demonstrate the key idea, in this section we restrict our attention to the subset of $\R^d$ which contains incoherent $x$ with $\ell_2$ norm bounded by 1, that is, we consider, 
\begin{align}
\mathcal{B} & = \Set{x: \|x\|_{\infty}  \le \frac{2\mu}{\sqrt{d}},  \|x\|\le 1} \mper\label{eqn:x_incoherent}
\end{align}

Note that the desired solution $z$ is in $\mathcal{B}$, and
the regularization $R(x)$ vanishes inside $\mathcal{B}$.

The following lemmas assume $x$ satisfies the first and second order \KKT conditions, and deduce a sequence of properties that $x$ must satisfy.

\begin{lemma}\label{lem:warmup_second-orderkkt}
			Under the setting of Theorem~\ref{thm:rank1_main}
		, with high probability over the choice of $\Omega$, 
for	any $x\in \mathcal{B}$ that satisfies second-order \KKT condition~\eqref{eq:kkt2}  we have, 		$$\|x\|^2 \ge 1/4.$$
The same is true if $x\in \mathcal{B}$ only satisfies $\tau$-relaxed second order \KKT condition for $\tau \le 0.1p$.
\end{lemma}

\begin{proof}
	We plug in $v = z$ in the second-order \KKT condition~\eqref{eq:kkt2},  and obtain that 
	\begin{equation}
	\Norm{\Po(zx^{\top} +xz^{\top})}_F ^2 \ge 2z^{\top} \Po (M-xx^{\top})z \mper\label{eqn:10}
	\end{equation}
	
	Intuitively, when restricted to $\Omega$, the squared Frobenius on the LHS and the quadratic form on the RHS should both be approximately a $p$ fraction of the unrestricted case. In fact, both LHS and RHS can be written as the sum of terms of the form $\inner{\Po(uv^T),\Po(st^T)}$, because 
\begin{align*}
	\Norm{\Po(zx^{\top} +xz^{\top})}_F ^2 &= 2\inner{\Po(zx^T),\Po(zx^T)}+2\inner{\Po(zx^T),\Po(xz^T)}\\
	 2z^{\top} \Po (M-xx^{\top})z &= 2\inner{\Po(zz^T),\Po(zz^T)} - 2\inner{\Po(xx^T),\Po(zz^T)}.
\end{align*}
	
Therefore we can use concentration inequalities (Theorem~\ref{thm:rank1_incohenrence_concentration}), and simplify the equation
	\begin{align}
	\textup{LHS of}~\eqref{eqn:10} &  = p  	\Norm{zx^{\top} +xz^{\top}}_F^2 \pm O(\sqrt{pd \|x\|_{\infty}^2\|z\|_{\infty}^2 \|x\|^2 \|z\|^2}) \nonumber\\
	&  = 2p\|x\|^2 \|z\|^2 + 2p\inner{x,z}^2 \pm O(p\epsilon) \mcom \tag{Since $x,z\in \mathcal{B}$} 
	 \end{align}
	 where $\epsilon = O(\mu^2 \sqrt{\frac{\log d}{pd}})$.	Similarly, by Theorem~\ref{thm:rank1_incohenrence_concentration} again, we have 
	 	\begin{align}
	 	\textup{RHS of}~\eqref{eqn:10} &  = 2\left(\inner{\Po(zz^{\top}), \Po(zz^{\top})} - \inner{\Po(xx^{\top}), \Po(zz^{\top})}\right)\tag{Since $M = zz^{\top}$}\\
	 	& = 2p\|z\|^4 - 2p\inner{x,z}^2  \pm O(p\epsilon) \tag{by Theorem~\ref{thm:rank1_incohenrence_concentration} and $x,z\in \mathcal{B}$}
	 	\end{align}
	 	
	 (Note that even we use the $\tau$-relaxed second order \KKT condition, the RHS only becomes $1.99p\|z\|^4 - 2p\inner{x,z}^2  \pm O(p\epsilon)$ which does not effect the later proofs.)
	 	
	 Therefore plugging in estimates above back into equation~\eqref{eqn:10}, we have that 
	 \begin{equation}
	 2p\|x\|^2 \|z\|^2 + 2p\inner{x,z}^2 \pm O(p\epsilon) \ge 2\|z\|^4 - 2\inner{x,z}^2  \pm O(p\epsilon)\mcom \nonumber
	 \end{equation}
	 which implies that $6p\|x\|^2\|z\|^2 \ge 2p\|x\|^2 \|z\|^2 + 4p\inner{x,z}^2\ge 2p\|z\|^4 -O(p\epsilon)$. Using $\|z\|^2 = 1$, and $\epsilon$ being sufficiently small,  we complete the proof. 
\end{proof}

Next we use first order \KKT condition to pin down another property of $x$ -- it has to be close to $z$ after scaling. Note that this doesn't mean directly that $x$ has to be close to $z$ since $x =0$ also satisfies first order \KKT condition (and therefore the conclusion~\eqref{eqn:9} below). 
\begin{lemma}\label{lem:warmup_first-orderkkt}
	With high probability over the randomness of $\Omega$, for any $x\in \mathcal{B}$ that satisfies first-order \KKT condition~\eqref{eq:kkt1}, 	we have that $x$ also satisfies 
	\begin{equation}
		\Norm{\inner{z,x} z - \|x\|^2 x }\le O(\epsilon)\mper\label{eqn:9}
	\end{equation}
	where $\epsilon = \tilde{O}(\mu^3 (pd)^{-1/2})$. 
\end{lemma}

\begin{proof}
	Note that since $x\in \mathcal{B}$, we have $R(x) = 0$. Therefore first-order \KKT condition says that \begin{equation}
	\Po (M -xx^{\top})x = \Po(zz^{\top})x - \Po(xx^{\top})x= 0\mper \label{eqn:18}
	\end{equation} 
Again, intuitively we hope $\Po(zz^T) \approx p zz^T$ and $\Po(xx^T)x \approx p\|x\|^2 x$. These are made precise by the concentration inequalities Lemma~\ref{lem:concentration-3} and Theorem~\ref{thm:concentration_2} respectively.
	
By Theorem~\ref{thm:concentration_2}, we have that with high probability over the choice of $\Omega$, for every $x\in \mathcal{B}$, 
		\begin{equation}
		\|\Po(xx^{\top})x - pxx^{\top}x\|_F\le p\eps \|x\|^3\le p\epsilon \label{eqn:19}
		\end{equation}
		where $\epsilon = \tilde{O}(\mu^3 (pd)^{-1/2})$. 
	Similarly, by Lemma~\ref{lem:concentration-3}, we have that for with high probability over the choice of $\Omega$, 
	\begin{equation}
	\Norm{\Po(zz^{\top}) - pzz^{\top}} \le \epsilon p \mper \nonumber
	\end{equation}
	for $\epsilon = \tilde{O}(\mu^2 (pd)^{-1/2})$. Therefore for every $x$, 
		\begin{equation}
		\Norm{\Po(zz^{\top})x - pzz^{\top}x} \le \epsilon p\|x\| \le \epsilon p\mper\label{eqn:21}
		\end{equation}
	Plugging in estimates~\eqref{eqn:21} and~\eqref{eqn:19} into equation~\eqref{eqn:18}, we complete the proof. 
\end{proof}

Finally we combine the two  \KKT conditions and show equation~\eqref{eqn:9} implies $xx^T$ must be close to $zz^T$.

\begin{lemma}\label{lem:warmup-close}
	Suppose vector $x$ satisfies that $\|x\|^2 \ge 1/4$, and that 
$
	\Norm{\inner{z,x} z - \|x\|^2 x }\le \delta\mper\nonumber
$
	Then for $\delta  \in (0,0.1)$, 
	\begin{equation}
	\Norm{xx^{\top} - zz^{\top}}_F^2\le O(\delta) \mper \nonumber
		\end{equation}
\end{lemma}
\begin{proof}
We write $z = u x + v$ where $u\in \R$ and $v$ is a vector orthogonal to $x$. Now we know $\inner{z,x}z = u^2\|x\|^2x + u\|x\|^2 v$, therefore
$$
\delta \ge \Norm{\inner{z,x} z - \|x\|^2 x } = \|x\|^2 \sqrt{u^2 \|v\|^2 + (1-u^2)^2}.
$$

In particular, we know $|1-u^2| \le 4\delta$ and $u\|v\| \le 4\delta$. This means $|u| \in 1\pm 3\delta$ and $\|v\| \le 8\delta$. Now we expand $xx^T - zz^T$:
$$
xx^T - zz^T = (1-u^2) xx^T + uxv^T+uvx^T + vv^T
$$

It is clear that all the terms have norm bounded by $O(\delta)$, therefore $\Norm{xx^{\top} - zz^{\top}}_F^2\le O(\delta)$.
\end{proof}

\subsection{Extension to general $x$} \label{subsec:nonincoherent}

\begin{figure}[]
	\centering
	\includegraphics[width=3in]{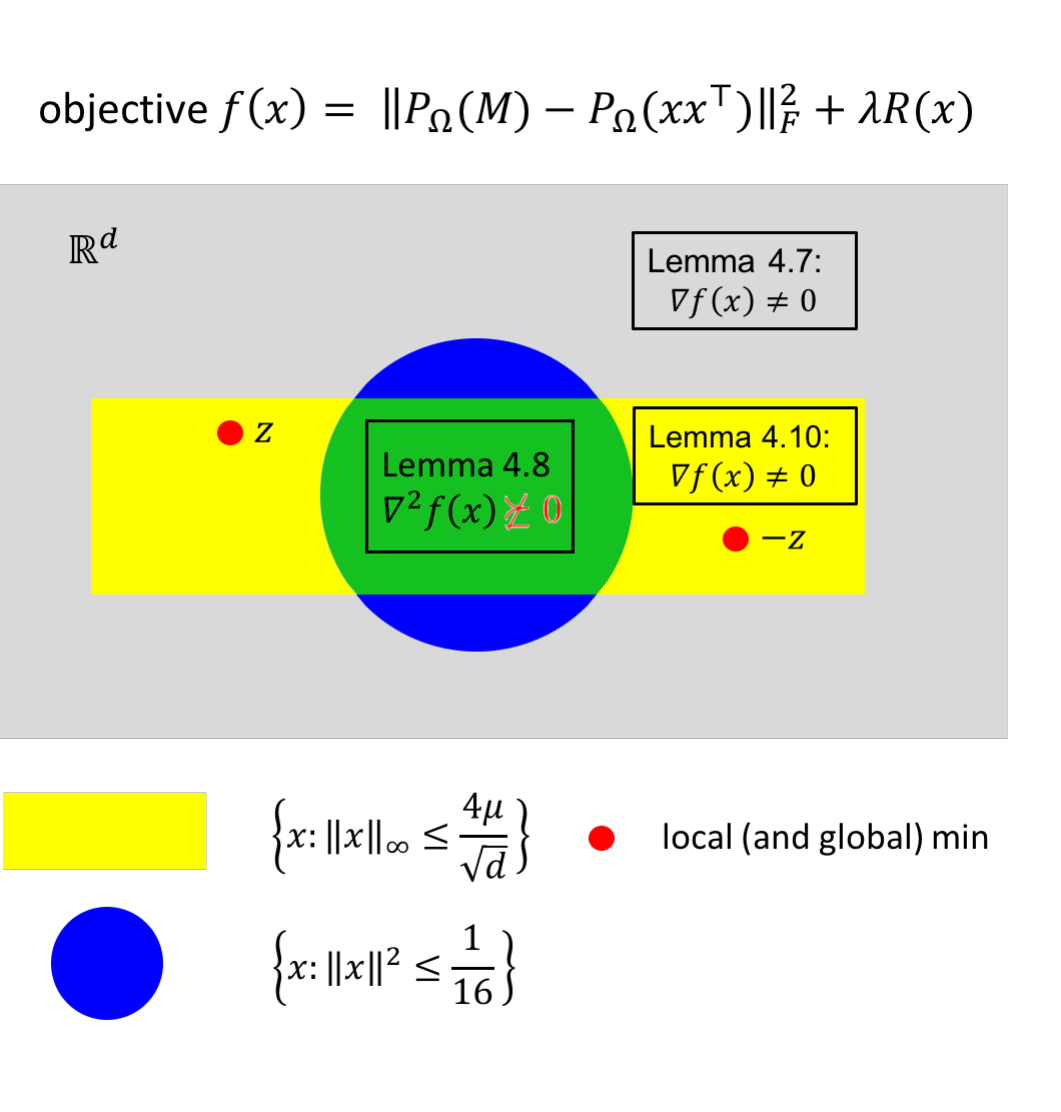}
	\caption{Partition of $\R^d$ into regions where our Lemmas apply. For example, Lemma 3.8 rules out the possibility that a point $x$ in the green region is local minimum. Here, The green region is the intersection of $\ell_{\infty}$ norm ball and $\ell_2$ norm ball. Both the white region and yellow region have non-zero gradient but for different reasons.} \label{fig:patition} 
\end{figure}

We have shown when $x$ is incoherent and satisfies first and second order \KKT conditions, then it must be close to $z$ or $-z$. Now we need to consider more general cases when $x$ may have some very large coordinates. Here the main intuition is that the first order \KKT condition with a proper regularizer is enough to guarantee that $x$ cannot have a entry that is too much bigger than $\mu/\sqrt{d}$. 
\begin{lemma}\label{lem:almost_incoherence}
		With high probability over the choice of $\Omega$, for any
	$x$ that satisfies first-order order \KKT condition~\eqref{eq:kkt1}, we have
	\begin{equation}
	\|x\|_{\infty}\le 4\max\left\{\alpha,\mu \sqrt{p/\lambda}\right\}\mper\label{eqn:almost_incoherence}
	\end{equation}
\end{lemma}
Here we recall that $\alpha$ was chosen to be $10\mu/\sqrt{d}$ and $\lambda$ is chosen to be large so that the $\alpha$ dominates the second term $\mu \sqrt{p/\lambda}$ in the setting of Theorem~\ref{thm:rank1_main}. 
\begin{proof}[Proof of Lemma~\ref{lem:almost_incoherence}]
	\newcommand{\imax}{i^{\star}}
	Suppose $\imax=\max_j |x_j|$.  Without loss of generality, suppose $x_{\imax} \ge 0$.  Suppose $\imax$-th row of $\Omega$ consists of entries with index $[i]\times S_{\imax}$. If $|x_{\imax}|\le 2\alpha$, we are done. Therefore in the rest of the proof we assume $|x_{\imax}|> 2\alpha$. Note that when $p \ge c(\log d)/d$ for sufficiently large constant $c$, with high probability over the choice of $\Omega$, we have $|S_{\imax}|\le 2pd$. In the rest of argument we are working with such an $\Omega$ with $|S_{\imax}|\le 2pd$. 
	
	We will compare the $\imax$-th coordinate of LHS and RHS of first-order \KKT condition~\eqref{eq:kkt1}. For preparation, we have
	\begin{align}
	\left|\left(\Po(M)x\right)_{\imax}\right| & = \left|\left(\Po(zz^{\top})x\right)_{\imax}\right| = \left|\sum_{j\in S_{\imax}} z_{\imax}z_jx_j\right|\nonumber\\
	& \le |x_{\imax}|\sum_{j\in S_{\imax}}|z_{\imax}z_j|\le |x_{\imax}|\cdot \mu^2/d\cdot |S_{\imax}|  \le 2|x_{\imax}|p \mu^2\label{eqn:1} 	\end{align}
	where the last step we used the fact that $|S_{\imax}|\le 2pd$. 
	Moreover, we have that 
	\begin{equation}
	(\Po(xx^{\top})x)_{\imax}= \sum_{j\in S_{\imax}} x_{\imax}x_j^2\ge 0\mcom\nonumber	\end{equation}
	and that 
	\begin{align}
	(\lambda\nabla R(x))_{\imax}& = 4\lambda(|x_{\imax}|-\alpha)^3\sign(x_{\imax}) \ge \frac{\lambda}{2} |x_{\imax}|^3 \tag{Since $x_{\imax}\ge 2\alpha$}
	\end{align}
	Now plugging in the bounds above into the $\imax$-th coordinate of equation~\eqref{eq:kkt1}, we obtain
	\begin{align}
	4|x_{\imax}|p \mu^2\ge 	2(\Po(M-xx^{\top})x)_{\imax}\ge 	(\lambda\nabla R(x))_{\imax}\ge \frac{\lambda}{2} |x_{\imax}|^3\mcom\nonumber
	\end{align}
	which implies that $|x_{\imax}|\le 4\sqrt{p\mu^2/\lambda}$. 
\end{proof}

Setting $\lambda \ge \mu^2 p/\alpha^2$ and $\alpha = 10\mu\sqrt{1/d}$, Lemma~\ref{lem:almost_incoherence} ensures that any $x$ that satisfies first-order \KKT condition is the following ball, 
\begin{equation}
\mathcal{B}' = \Set{x\in \R^d : \|x\|_{\infty}\le 4\alpha }\mper \nonumber
\end{equation}

Then we would like to continue to use arguments similar to Lemma~\ref{lem:warmup_second-orderkkt} and \ref{lem:warmup_first-orderkkt}. However, things have become more complicated as now we need to consider the contribution of the regularizer.

\begin{lemma}[Extension of Lemma~\ref{lem:warmup_second-orderkkt}]\label{lem:general_x_2norm}
			In the setting of Theorem~\ref{thm:rank1_main}, 
	with high probability over the choice of $\Omega$, suppose $x\in \mathcal{B}'$ satisfies second-order \KKT condition~\eqref{eq:kkt2} or $\tau$-relaxed condition for $\tau \le 0.1p$, we have $\|x\|^2 \ge 1/8$.
\end{lemma}

The guarantees and proofs are very similar to Lemma~\ref{lem:warmup_second-orderkkt}. The main intuition is that we can restrict our attentions to coordinates whose regularizer is equal to 0. See Section~\ref{sec:rank1proof} for details.
	
We will now deal with first order \KKT condition. We first write out the basic extension of Lemma~\ref{lem:warmup_first-orderkkt}, which follows from the same proof except we now include the regularizer term.

\begin{lemma}[Basic extension of Lemma~\ref{lem:warmup_first-orderkkt}]
\label{lem:first-order-general-basic}
	With high probability over the randomness of $\Omega$, for any $x\in \mathcal{B}'$ that satisfies first-order \KKT condition~\eqref{eq:kkt1}, we have that $x$ also satisfies 
	\begin{equation}
	\Norm{\inner{z,x} z - \|x\|^2 x - \gamma \cdot \nabla R(x) }\le O(\epsilon)\mper\label{eqn:29}
	\end{equation}
	where $\epsilon = \tilde{O}(\mu^6 (pd)^{-1/2})$ and $\gamma = \lambda/(2p)\ge 0$. \end{lemma}

Next we will show that we can remove the regularizer term, the main observation here is nonzero entries $\nabla R(x)$ all have the same sign as the corresponding entries in $x$. See Section~\ref{sec:rank1proof} for details.

\begin{lemma}\label{lem:rank-1-reduction}
	Suppose $x\in \mathcal{B}'$ satisfies that $\|x\|^2\ge 1/8$, under the same assumption as Lemma~\ref{lem:first-order-general-basic}. we have, 		\begin{equation}
	\Norm{\inner{x,z}z-\|x\|^2x} \le O(\epsilon) \nonumber
	\end{equation}
\end{lemma}

Finally we combine Lemma~\ref{lem:almost_incoherence}, Lemma~\ref{lem:general_x_2norm}, Lemma~\ref{lem:rank-1-reduction} and Lemma~\ref{lem:warmup-close} to prove Lemma~\ref{thm:warmup-close}. The argument are also summarized in Figure~\ref{fig:patition}, where we partition $\R^d$ into regions where our lemmas apply.

\section{Rank-r case}

\label{sec:rankr}

In this section we show how to extend the results in Section~\ref{sec:rank1} to recover matrices of rank $r$. Here we still use the same proof strategy of Section~\ref{sec:intuition}. Though for simplicity we only write down the proof for the partial observation case,  while the analysis for the full observation case (which was our starting point) can be obtained by substituting $[d]\times [d]$ for $\Omega$ everywhere. 

Recall that in this case we assume the original matrix $M = ZZ^T$, where $Z \in \R^{d\times r}$. We also assume Assumption~\ref{assump:incoherence}. The objective function is very similar to the rank 1 case
\begin{equation}
f(X) = \frac12 \Norm{\Po(M-XX^{\top})}_F^2 + \lambda R(X)\mcom \label{eqn:objective_rank_r}
\end{equation}
where 
$
	R(X) = \sum_{i=1}^d r(\|X_i\|)\mper
$
Recall that $r(t) = (|t|-\alpha)^4 \Ind_{t \ge \alpha}$. Here $\alpha$ and $\lambda$ are again parameters that we will determined later.

Without loss of generality, we assume that $\|Z\|_F^2 = r$ in this section. This implies that $\smax(Z)\ge 1\ge \smin(Z)$. 
Now we shall state the first and second order \KKT conditions:

\begin{proposition}\label{prop:kkt}
	If $X$ is a local optimum of objective function~\eqref{eqn:objective_rank_r}, its first order \KKT condition is, 
	\begin{align}
	2\Po (M)X &=2\Po (XX^{\top})X +\lambda \nabla R(X)\label{eq:kkt1-rankr}\mcom
	\end{align}
	and the second order \KKT condition is equivalent to 
	\begin{align}
	\forall V\in \R^{d\times r}, ~\norm{\Po(VX^{\top} +XV^{\top})}_F ^2 + \lambda \inner{V^{\top}, \nabla^2 R(X)V} & \ge 2\inner{\Po (M-XX^{\top}), VV^{\top}} \mper
	\label{eq:kkt2-rankr}
	\end{align}
\end{proposition}

Note that the regularizer now is more complicated than the one dimensional case, but luckily we still have the following nice property.

\begin{proposition}\label{prop:gradient}
	We have that $
	\nabla R(X) = \Gamma X$ where $\Gamma \in \R^{d \times d}$ is a diagonal matrix with $\Gamma_{ii} = \frac{4(\|X_i\|-\alpha)^4}{\|X_i\|}\Ind_{\|X_i\|\ge \alpha}$. As a direct consequence, $\inner{(\nabla R(X))_i, X_i} \ge 0$ for every $i\in [d]$. 
\end{proposition}

Now we are ready to state the precise version of Theorem~\ref{thm:main}:

\begin{theorem}\label{thm:main-rank-r}
	 Suppose $p \ge C\max\{\mu^6\kappa^{16} r^4, \mu^4\kappa^4r^6\} d^{-1}\log^2 d$ where $C$ is a large enough constant.  Let $\alpha = 32\mu \kappa r/\sqrt{d}, \lambda \ge \mu^2 rp/\alpha^2$. Then with high probability over the randomness of $\Omega$, any local minimum $X$ of $f(\cdot)$ satisfies that $f(X) = 0$, and in particular, $ZZ^{\top} = XX^{\top}$. 
	 
	 Moreover, If $X$ satisfies that $\|\nabla f(X)\|_F \le \delta\le p\sigma^{3}(Z)/C$ and $\nabla^2 f(X)\succeq -1/C \cdot \mu^2 \kappa r^2 p^{1/2}d^{-1/2} I$, then $X$ is an approximate global minimum in the sense that $\|XX^\top - M\|_F^2\le O(\delta/p).$
\end{theorem}

The proof of this Theorem follows from a similar path as Theorem~\ref{thm:rank1_main}. We first notice that because of the regularizer, any matrix $X$ that satisfies first order \KKT condition must be somewhat incoherent (this is analogues to Lemma~\ref{lem:almost_incoherence}):

\begin{lemma}\label{lem:almost_incoherence_rankr}
	Suppose $|S_i| \le 2pd$. Then for any $X$ satisfies 1st order \KKT~\eqref{eq:kkt1-rankr}, we have 
	\begin{equation}
	\|X\|_{2\rightarrow \infty} = \max_{i}\|X_i\| \le 4\max\left\{\alpha, \mu \sqrt{rp/\lambda}\right\} \label{eqn:almost_incoherent_rankr}
	\end{equation}
\end{lemma}
\newcommand{\istar}{i^{\star}}
\begin{proof}
	Assume $i^{\star} = \argmax_i \|X_i\|$. Suppose the $i$th row of $\Omega$ consists of entries with index $[i]\times S_i$. If $\|X_{\istar}\|\le 2\alpha$, then we are done. Therefore in the rest of the proof we assume $\|X_{\istar}\|\ge 2\alpha$. 
	
	We will compare the $i$-th row of LHS and RHS of~\eqref{eq:kkt1-rankr}. For preparation, we have
	\begin{align}
\left(\Po(M)x\right)_{\istar} & = \left(\Po(ZZ^{\top})X\right)_{\istar} = \left(\Po(ZZ^{\top})\right)_{\istar} X	\end{align}
Then we have that 
\begin{align}
\Norm{\left(\Po(ZZ^{\top})\right)_{\istar}}_1 & = \sum_{j\in S_{\istar}} |\inner{Z_{\istar}, Z_j}| \nonumber\\
& \le \sum_{j\in S_{\istar}} \norm{Z_{\istar}}\norm{Z_j}\le \sum_{j\in S_{\istar}} \mu^2 r/d |S_1| \tag{by incoherence of $Z$} \\
& \le 2\mu^2rp \mper\tag{by $|S_{\istar}|\le 2pd$}
\end{align}

Therefore we can bound the $\ell_2$ norm of LHS of 1st order \KKT condition~\eqref{eq:kkt1-rankr} by 
\begin{align}
\Norm{\left(\Po(ZZ^{\top})X\right)_{\istar}} & \le \Norm{\left(\Po(ZZ^{\top})\right)_{\istar}}_1 \Norm{X^{\top}}_{1\rightarrow 2} \nonumber\\
& \le 2\mu^2rp \Norm{X}_{2\rightarrow \infty} \tag{by $\Norm{X}_{2\rightarrow \infty} =\Norm{X^{\top}}_{1\rightarrow 2}$ } \\
& = 2\mu^2rp \Norm{X_{\istar}} \label{eqn:eqn12}
\end{align}

Next we lowerbound the norm of the RHS of equation~\eqref{eq:kkt1-rankr}. 
	We have that 
	\begin{equation}
	(\Po(XX^{\top})X)_{\istar} = \sum_{j\in S_{\istar}} \inner{X_{\istar}, X_j}X_j = X_i \sum_{j\in X_{\istar}}X_j^{\top}X_j\mcom\nonumber	\end{equation}
	which implies that 
	\begin{equation}
		\inner{(\Po(XX^{\top})X)_{\istar}, X_{\istar}} = X_{\istar}\left(\sum_{j\in X_{\istar}}X_j^{\top}X_j\right)X_{\istar}^{\top}\ge 0\mper\label{eqn:eqn17}
	\end{equation}
	Using Proposition~\ref{prop:gradient} we obtain that 
		\begin{equation}
		\inner{(\Po(XX^{\top})X)_{\istar}, (\nabla R(X))_{\istar}} = \Gamma_{ii} X_{\istar}\left(\sum_{j\in X_{\istar}}X_j^{\top}X_j\right)X_{\istar}^{\top}\ge 0\mper\label{eqn:eqn21}
		\end{equation}
It follows that 
	\begin{align}
	\Norm{(\Po(XX^{\top})X)_{\istar} + (\lambda \nabla R(X))_{\istar}}&\ge \Norm{ (\lambda \nabla R(X))_{\istar}}  \tag{by equation~\eqref{eqn:eqn21}}\\
	& = \frac{4\lambda(\Norm{X_{\istar}}-\alpha)^3}{\Norm{X_{\istar}}} \cdot \|X_{\istar}\| \tag{by Proposition~\ref{prop:gradient}} \\
	& \ge \frac{\lambda}{2}\|X_{\istar}\|^3 \tag{by the assumptino $\|X_{\istar}\| \ge 2\alpha$}
	\end{align}
	Therefore plugging in equation above and equation~\eqref{eqn:eqn12} into 1st order \KKT condition~\eqref{eq:kkt1-rankr}. We obtain that $\|X_{\istar}\|\le \sqrt{8\mu^2 rp/\lambda}$ which completes the proof.
\end{proof}

Next, we prove a  property implied by first order \KKT condition, which is similar to Lemma~\ref{lem:first-order-general-basic}.

\begin{lemma}\label{lem:Xnormbound}
		In the setting of Theorem~\ref{thm:main-rank-r}, with high probability over the choice of $\Omega$, for any $X$ that satisfies 1st order \KKT condition~\eqref{eq:kkt1-rankr}, we have
	\begin{equation}
	\|X\|_F^2\le 2r\smaxz^2\mper\label{eqn:Xnormbound}
	\end{equation}
	Moreover, we have 
	\begin{equation}
	\smax(X)\le 2\smaxz r^{1/6}\mper\label{eqn:Xspectralnormbound}
	\end{equation}
	and 
	\begin{equation}
	\Norm{ZZ^TX-XX^TX- \gamma \nabla R(X)}_F\le O(\delta) \label{eqn:eqn11}
	\end{equation}
	where $\delta = O(\mu^3 \kappa^3 r^2 \log^{0.75}(d)\smaxz^{-3}(dp)^{-1/2})$ and $\gamma = \lambda/(2p)\ge 0$. 
\end{lemma}

\begin{proof}
	If $\|X\|_F \le \sqrt{r\smaxz^2}$ we are done. When $\|X\|_F \ge \sqrt{r\smaxz^2}$, by Lemma~\ref{lem:almost_incoherence_rankr}, we have that $\max\|X_i\|\le 4\alpha = O(\mu \kappa r/\sqrt{d})$, and therefore $\max \|X_i\| \le \nu \|X\|_F$ with $\nu = O(\mu\kappa \sqrt{r}/\smaxz)$. Then by Theorem~\ref{thm:concentration_2},  we have that 
	\begin{equation}
		\Norm{\Po(ZZ^{\top})X-pZZ^{\top}X}_F \le p\delta\mcom \nonumber
	\end{equation}
	and 
		\begin{equation}
		\Norm{\Po(XX^{\top})X -pXX^{\top}X}_F \le p\delta\mcom \nonumber
		\end{equation}
	where  $\delta = O(\mu^3 \kappa^3 r^2 \log^{0.75}(d)\smaxz^{-3}(dp)^{-1/2})$. These two imply equation~\eqref{eqn:eqn11}. Moreover,  we have 
	\begin{align}
	p\Norm{ZZ^{\top}X}_F= \Norm{\Po(ZZ^{\top})X}_F \pm p\delta & = \Norm{\Po(XX^{\top})X + \lambda R(X)}_F \pm p\delta \tag{by equation~\eqref{eq:kkt1-rankr}} \\
	& \ge \Norm{\Po(XX^{\top})X}_F	\pm p\delta \tag{by equation~\eqref{eqn:eqn21}}\\
		& \ge p\Norm{XX^{\top}X}_F	\pm 2p\delta \label{eqn:eqn22}
	\end{align}
	Suppose $X$ has singular value $\sigma_1\ge \dots \ge \sigma_r$. Then we have $\Norm{ZZ^{\top}X}_F^2\le \|ZZ^{\top}\|^2\|X\|_F^2\le \smaxz^4 \|X\|_F^2 = \smaxz^4 (\sigma_1^2+\dots+\sigma_r^2)$. On the other hand, $\Norm{XX^{\top}X}_F^2= \sigma_1^6+\dots+\sigma_r^6$. Therefore, equation~\eqref{eqn:eqn22} implies that 
	
	\begin{equation}
	(1+O(\delta))\smaxz^4 \sum_{i=1}^{r}\sigma_i^2 \ge \sum_{i=1}^r\sigma_i^6 \nonumber
	\end{equation}
		Then we have (by Proposition~\ref{prop:2-6norm}) we complete the proof.

	\end{proof}

Now we look at the second order \KKT condition, this condition implies the smallest singular value of $X$ is large (similar to Lemma~\ref{lem:general_x_2norm}). Note that this lemma is also true even if $x$ only satisfies relaxed second order \KKT condition with $\tau = 0.01p\sminz$.
\begin{lemma}\label{lem:Xnormlowerbound}
		In the setting of Theorem~\ref{thm:main-rank-r}. With high probability over the choice of $\Omega$, 
	suppose $X$ satisfies equation~\eqref{eqn:Xnormbound},~\eqref{eqn:almost_incoherent_rankr} the 2nd order \KKT condition~\eqref{eq:kkt2-rankr}. Then, 
	\begin{equation}
	\smin(X) \ge \frac{1}{4}\smin(Z)\label{eqn:Xnormlowerbound}
	\end{equation}
\end{lemma}

\begin{proof}
	Let $J = \{i: \|X_i\|\le \alpha\}$. Let $v\in \R^r$ such that $\|Xv\| = \sigma_{\min}(X)$. .  Let $Z_J$ be the matrix that has the same $i$-th row as $Z$ for every $i\in J$ and 0 elsewhere. 
	
	We claim that $\smin(Z_J) \ge \frac{1}{2}\sminz$. Let $L = [d]-J$. Since for any $i\in L$ it holds that $\|X_i\|\ge \alpha$, we have $|L|\alpha^2 \le \|X\|_F^2 \le 2r\smaxz^2$ (by equation~\eqref{eqn:Xnormbound}), and it follows that $|L|\le 2r\smaxz^2/\alpha^2$. Therefore, 
	\begin{align}
	\smin(Z_J) &\ge \smin(Z) - \smax(Z_{L})  \ge \smin(Z) -  \|Z_L\|_F\nonumber\\
	& \ge \smin(Z) - \sqrt{|L|r\mu^2/d} \ge \sminz - \sqrt{2r^2\smaxz^2\mu^2/(\alpha^2d)} \nonumber\\
	& \ge \frac{1}{2}\sminz \tag{by $\alpha\ge \frac{r\kappa\mu }{\sqrt{d}}$}\mper
	\end{align}
	Therefore, $Z_J$ has column rank exactly $r$. By variational characterization of singular values, we have that for there exists unit vector $z_J \in \colspan(Z_J)$ such that $\|X^{\top}z_J\|\le \sigma_{\min}(X)$. 
	Since $z_J\in \colspan(Z_J)$ is a unit vector, we have that $z_J$ can be written as $z_J = Z_J\beta$ where $\|\beta\|\le \frac{1}{\smin(Z_J)}\le O(1/\sminz).$ Therefore this in turn implies that $\|z_J\|_{\infty}\le \|Z_J\|_{2\rightarrow \infty}\|\beta\| \le O(\mu \sqrt{r/d}/\sminz) \le O(\mu \kappa\sqrt{r/d})$. 
	
	We will plug in $V = z_Jv^T$ in the 2nd order \KKT condition~\eqref{eq:kkt2-rankr}. 
	Note that since $z_J\in \colspan(Z_J)$, it is supported on subset $J$, and therefore $\nabla^2 R(X) V  = 0$. Therefore the term about regularization in~\eqref{eq:kkt2-rankr} will vanish.  For simplicity, let $y = X^{\top}z_J$, $w = Xv$
	We obtain that taking $V = z_Jv^{\top}$ in equation~\eqref{eq:kkt2-rankr} will result in		\begin{equation}
	\Norm{\Po(wz_J^{\top}+z_Jw^{\top})}_F^2 \ge 2 \inner{\Po(ZZ^{\top}-XX^{\top}), z_Jz_J^{\top}} \nonumber	\end{equation} 	Note that we have that $\|w\|_{\infty} \le \|X\|_{2\rightarrow \infty} \|v\|\le \mu \sqrt{r/d}$.  Recalling that $\|z_J\|_{\infty}\le O(\mu \kappa\sqrt{r/d})$, by Theorem~\ref{thm:rank1_incohenrence_concentration}, we have that 
	\begin{equation}
	p\Norm{wz_J^{\top}+z_Jw^{\top}}_F^2 \ge 2p\inner{ZZ^{\top}-XX^{\top}, z_Jz_J^{\top}} - \delta p \nonumber
	\end{equation}
	where $\delta = O(\mu^2 \kappa r^2 (pd)^{-1/2})$. Then simple algebraic manipulation gives that 
	\begin{equation}
	\inner{w,z_J}^2 + \|w\|^2 \|z_J\|^2 + \|X^{\top}z_J\|^2 \ge \|Z^{\top}z_J\|^2 - \delta/2\label{eqn:eqn19}
		\end{equation}
	Note that $\inner{w,z_J} =  \inner{v, X^{\top}z_J} = \inner{y,v}$. Recall that $\|z_J\|= 1$ and $z\in \colspan(Z_J)$, and therefore $\|Z^{\top}z_J\| = \|Z_J^{\top}z_J\|\ge \smin^2(Z_J)$. Moreover, recall that $\|y\| = \|X^{\top}z_J\|\le \smin(X)$. Using these with equation~\eqref{eqn:eqn19} we obtain that 
	\begin{align}
	\inner{w,z_J}^2 + \|w\|^2 \|z_J\|^2 + \|X^{\top}z_J\|^2 & \le \inner{y,v}^2 + \|w\|^2 + \|y\|^2\nonumber\\
	& \le 2\|y\|^2 + \smin^2 (X) \tag{by Cauchy-Schwarz and $\|w\| = \smin(X)$. }\\
	& \le 3\smin^2 (X) \tag{by $\|y\| \le \smin(X).$ }
	\end{align}
	Therefore together with equation~\eqref{eqn:eqn19} and $\|Z^{\top}z_J\| \ge \smin^2(Z_J)$ we obtain that \begin{equation}
	\smin(X) \ge (1/2-\Omega(\delta))\smin(Z_J) \label{eqn:eqn20} 
	\end{equation}
		Therefore combining equation~\eqref{eqn:eqn20}  and the lower bound on $\smin(Z_J)$ we complete the proof. 
\end{proof}

Similar as before, we show it is possible to remove the regularizer term here, again the intuition is that the regularizer is always in the same direction as $X$.

\begin{lemma}\label{prop:reduction}
	Suppose $X$ satisfies equation~\eqref{eqn:almost_incoherent_rankr} and~\eqref{eqn:Xnormlowerbound} and~\eqref{eqn:Xspectralnormbound}, then for any $\gamma \ge 0$, 
	\begin{equation}
	\Norm{ZZ^TX-XX^TX}_F^2 \le \Norm{ZZ^TX-XX^TX- \gamma \nabla R(X)}_F^2
	\label{eq:gradientformrankr}
	\end{equation}
\end{lemma}

\begin{proof}
	Let $L = \{i: \|X_i\| \ge \alpha\}$. For $i\not\in L$, we have that $(\nabla R(X))_i = 0$. Therefore it suffices to prove that for every $i\in L$, 
	\begin{equation}
	\Norm{Z_iZ^{\top}X-X_iX^{\top}X}^2 \le \Norm{Z_iZ^{\top}X-X_iX^{\top}X- (\gamma \nabla R(X))_i}^2\nonumber
	\end{equation}
		It suffices to prove that 
	\begin{equation}
		\inner{(\nabla R(X))_i, X_iX^{\top}X-Z_iZ^{\top}X} \ge 0\label{eqn:eqn24}
	\end{equation}
	By proposition~\ref{prop:gradient}, we have $\nabla R(X))_i = \Gamma_{ii} X_i$ for $\Gamma_{ii} \ge 0$. Then 
		\begin{align}
		\inner{(\nabla R(X))_i, X_iX^{\top}X} & = \Gamma_{ii} \inner{X_i, X_iX^{\top}X} = \Gamma_{ii} X_iX^{\top}XX_i^{\top}\nonumber\\
		& \ge \Gamma_{ii} \Norm{X_i}^2\smin(X^TX) \nonumber\\
		& \ge \frac{1}{16}\Gamma_{ii}\Norm{X_i}^2 \sminz^2 \tag{by equation~\ref{eqn:Xnormlowerbound}}
		\end{align}
	On the other hand, we have 
	\begin{align}
	\inner{(\nabla R(X))_i, Z_iZ^{\top}X} &= \Gamma_{ii}\inner{X_i,Z_iZ^{\top}X} \nonumber\\
	& \le \Gamma_{ii} \|X_i\| \|Z_i\|\smax(Z^TX) \le \Gamma_{ii} \|X_i\| \|Z_i\|\smaxz\smax(X) \nonumber\\
	& 	\le 2\Gamma_{ii} \|X_i\| \|Z_i\|\smaxz^2 r^{1/6}\tag{by equation~\eqref{eqn:Xspectralnormbound}}\\
	& \le \frac{1}{16}\Gamma_{ii} \|X_i\|^2\sminz^2 r^{-1/3}\tag{because $\|X_i\|\ge \alpha \ge 32 \kappa\mu r/\sqrt{d} \ge 32\sqrt{r} \norm{Z_i}$}
	\end{align}

	Therefore combining two equations above we obtain equation~\eqref{eqn:eqn24} which completes the proof.  
\end{proof}

Finally we show the form in Equation \eqref{eq:gradientformrankr} implies $ZZ^T$ is close to $XX^T$ (this is similar to Lemma~\ref{lem:warmup-close}).

\begin{lemma}\label{prop:1}
	Suppose $X$ and $Z$ satisfies that $\smin(X) \ge 1/4 \cdot \smin(Z)$ and that 
	\begin{equation}
	\Norm{ZZ^TX-XX^TX}_F^2 \le \delta^2\nonumber
	\end{equation}
	where $\delta \le \sigma_{min}^3(Z)/C$ for a large enough constant $C$, then 
	\begin{equation}
	\|XX^{\top} -ZZ^{\top}\|_F^2 \le O(\delta \kappa^2/\sigma_{min}(Z)).\nonumber
	\end{equation}
\end{lemma}

\begin{proof}
The proof is similar to the one-dimensional case, we will separate $Z$ into the directions that are in column span of $X$ and its orthogonal subspace. We will then show the projection of $Z$ in the column span is close to $X$, and the projection on the orthogonal subspace must be small.

Let $Z = U+V$ where $U = \mbox{Proj}_{span(X)} Z$ is the projection of $Z$ to the column span of $X$, and $V$ is the projection to the orthogonal subspace. Then since $V^TX = 0$ we know
$$
ZZ^TX = (U+V)(U+V)^T X = UU^T X + VU^TX.
$$
Here columns of the first term $UU^TX$ are in the column span of $X$, and the columns second term $VU^TX$ are in the orthogonal subspace. Therefore,
$$
\|ZZ^TX-XX^TX\|_F^2 = \| UU^T X - XX^TX\|_F^2 + \|VU^TX\|_F^2 \le \delta^2.
$$

In particular, both terms should be bounded by $\delta^2$. Therefore $\|UU^T - XX^T\|_F^2 \le \delta^2/\sigma_{min}^2(X) \le 16\delta^2/\sigma_{min}^2(Z)$. 

Also, we know $\sigma_{min}(UU^TX) \ge \sigma_{min}(XX^TX) - \delta \ge \sigma_{min}(Z)^3/128$ if $\delta \le \sigma_{min}(Z)^3/128$. Therefore $\sigma_{min}(U^TX)$ is at least $\sigma_{min}(Z)^3/\|Z\|128$. Now $\|V\|_F^2 \le \delta^2/\sigma_{min}(U^TX)^2 \le O(\delta^2\|Z\|^2/\sigma_{min}(Z)^6)$.

Finally, we can bound $\|UV^T\|_F$ by $\|U\|\|V\|_F \le \|Z\| \|V\|_F$ (last inequality is because $U$ is a projection of $Z$), which at least $\Omega(\|V\|_F^2)$ when $\delta \le \sigma_{min}(Z)^3/128$, therefore
$$
\|ZZ^T-XX^T\|_F \le \|UU^T-XX^T\|_F + 2\|UV^T\|_F + \|VV^T\|_F \le O(\delta\|Z\|^2/\sigma_{min}(Z)^3).
$$

\end{proof}

Last thing we need to prove the main theorem is a result from Sun and Luo\cite{sun2015guaranteed}, which shows whenever $XX^T$ is close to $ZZ^T$, the function is essentially strongly convex, and the only points that have $0$ gradient are points where $XX^T = ZZ^T$, this is explained in Lemma~\ref{lem:exact}. Now we are ready to prove Theorem~\ref{thm:main-rank-r}:

\begin{proof}[Proof of Theorem~\ref{thm:main-rank-r}]
	Suppose $X$ satisfies 1st and 2nd order \KKT condition. Then by Lemma~\ref{lem:Xnormbound} and Lemma~\ref{lem:almost_incoherence_rankr}, we have that $X$ satisfies equation~\eqref{eqn:almost_incoherent_rankr}, ~\eqref{eqn:Xnormbound},~\eqref{eqn:Xspectralnormbound} and~\eqref{eqn:eqn11}. Then by Lemma~\ref{lem:Xnormlowerbound}, we obtain that $\smin(X)\ge 1/6\cdot \sminz$. Now by Lemma~\ref{prop:reduction} and equation~\eqref{eqn:eqn11}, we have that $	\Norm{ZZ^TX-XX^TX}_F \le \delta 	$ for $\delta \le c \sminz^3/\kappa^2 $ for sufficiently small constant $c$.  Then by Lemma~\ref{prop:1} we obtain that $\|ZZ^{\top}-XX^{\top}\|_F \le c\sminz^2$ for sufficiently small constant $c$. By Lemma~\ref{lem:exact}, in this region the only points that satisfy the first order \KKT condition must satisfy $XX^T = ZZ^T$.	
	\end{proof}

\paragraph{Handling Noise}

To handle noise, notice that we can only hope to get an approximate solution in presence of noise, and to get that our Lemmas only depend on concentration bounds which still apply in the noisy setting. See Section~\ref{sec:noise} for details.

\section{Conclusions}

Although the matrix completion objective is non-convex, we showed the objective function has very nice properties that ensures the local minima are also global. This property gives guarantees for many basic optimization algorithms. An important open problem is the robustness of this property under different model assumptions: Can we extend the result to handle asymmetric matrix completion? Is it possible to add weights to different entries (similar to the settings studied in \cite{li2016recovery})? Can we replace the objective function with a different distance measure rather than Frobenius norm (which is related to works on 1-bit matrix sensing \cite{davenport20141})? We hope this framework of analyzing the geometry of objective function can be applied to other problems.

\bibliographystyle{alpha} 
\bibliography{ref}

\newpage

\appendix
\section{Omitted Proofs in Section~\ref{sec:rank1}}
\label{sec:rank1proof}
We first prove the equivalent form of the first and second order \KKT conditions:

\begin{lemma} [Proposition~\ref{prop:kktrank1} restated]
The first order \KKT condition of objective~\eqref{eqn:rank-1-objective} is, 
\begin{align}
2\Po (M -xx^{\top})x &=\lambda \nabla R(x)\mcom\nonumber
\end{align}
and the second order \KKT condition requires:
\begin{align}
\forall v\in \R^d, ~\norm{\Po(vx^{\top} +xv^{\top})}_F ^2 + \lambda v^{\top}\nabla^2 R(x)v&\ge 2v^{\top} \Po (M-xx^{\top})v \mper\nonumber
\end{align}
Moreover, The $\tau$-relaxed second order \KKT condition requires
\begin{align}
\forall v\in \R^d, ~\norm{\Po(vx^{\top} +xv^{\top})}_F ^2 + \lambda v^{\top}\nabla^2 R(x)v&\ge 2v^{\top} \Po (M-xx^{\top})v - \tau\|v\|^2 \mper\nonumber
\end{align}
\end{lemma}

\begin{proof}

We take the Taylor's expansion around point $x$. Let $\delta$ be an infinitesimal vector, we have
\begin{align*}
f(x+\delta) & = \frac12 \norm{\Po( M-(x+\delta)(x+\delta)^{\top})}_F ^2 + \lambda R(x+\delta) + o(\|\delta\|^2)\\
& =\frac12 \norm{\Po( M-xx^\top - (x\delta^\top + \delta x^\top) - \delta\delta^\top)}_F^2 +\lambda \left(R(x) + \inner{\nabla R(x),\delta} + \frac12 \delta^T \nabla^2 R(x)\delta\right) + o(\|\delta\|^2)\\
& = \frac12 \norm{M-xx^\top}_\Omega^2 + \lambda R(x)\\
& \quad - \inner{\Po(M-xx^\top), x\delta^\top + \delta x^\top} + \inner{\nabla R(x), \delta} + o(\|\delta\|^2)\\
& \quad - \inner{\Po(M-xx^\top), \delta\delta^\top} + \frac12 \|\Po(x\delta^\top + \delta x^\top) \|_F^2 + \frac{1}{2} \lambda  \delta^\top \nabla^2 R(x) \delta + o(\|\delta\|^2).
\end{align*}

By symmetry $\inner{\Po(M-xx^\top), x\delta^\top} = \inner{\Po(M-xx^\top), \delta x^\top} = \inner{\Po(M-xx^\top)x,\delta}$, so the first order \KKT condition is $\forall \delta, \inner{-2\Po(M-xx^\top)x + \lambda \nabla R(x),\delta} = 0$, which is equivalent to that $2\Po(M-xx^\top)x = \lambda \nabla R(x)$.

The second order \KKT condition says $- \inner{\Po(M-xx^\top), \delta\delta^\top} + \frac12 \|x\delta^\top + \delta x^\top\|_F^2 + \frac{1}{2} \lambda  \delta^\top \nabla^2 R(x) \delta \ge 0$ for every $\delta$, which is exactly equivalent to Equation \eqref{eq:kkt2}.
\end{proof}

Next we show the full proof for the second order \KKT condition:

\begin{lemma}[Lemma~\ref{lem:general_x_2norm} restated]
			In the setting of Theorem~\ref{thm:rank1_main}, 
	with high probability over the choice of $\Omega$, suppose $x\in \mathcal{B}'$ satisfies second-order \KKT condition~\eqref{eq:kkt2} or $\tau$-relaxed condition for $\tau \le 0.1p$, we have $\|x\|^2 \ge 1/8$.
\end{lemma}

\begin{proof}	If $\|x\|\ge 1$, then we are done. Therefore in the rest of the proof we assume $\|x\|\le 1$. 
	The proof is very similar to Lemma~\ref{lem:warmup_second-orderkkt}. We plug in $v = z_J$ instead into equation~\eqref{eq:kkt2}, where $J = \{i: |x_i|\le \alpha\}$. Note that $z_J^\top \nabla^2 R(z_J)z_J$ vanishes. We plug in $v = z_J$ in the equation~\eqref{eq:kkt2} and obtain that $x$ satisfies that 
	\begin{equation}
	\Norm{\Po(z_Jx^{\top} +xz_J^{\top})}_F ^2 \ge 2z_J^{\top} \Po (M-xx^{\top})z_J \mper\label{eqn:15}
	\end{equation}
	
		Note that we assume $\|x\|_{\infty}\le 2\alpha $, and in the beginning of the proof we assume wlog $\|x\|\le 1$. Moreover, we have $\|z_J\|\le \frac{\mu}{\sqrt{d}}$ an, $\|z_J\|\le 1$. Similarly to the derivation in the proof of Lemma~\ref{lem:warmup_second-orderkkt}, we apply Theorem~\ref{thm:rank1_incohenrence_concentration} (twice) and obtain that with high probability over the choice of $\Omega$, for every $x$, for $\epsilon = \tilde{O}(\mu^2(pd)^{-1/2})$, 
	\begin{align}
	\textup{LHS of}~\eqref{eqn:15} &  = p  	\Norm{z_Jx^{\top} +xz_J^{\top}}_F^2 \pm O(p\epsilon)  = 2p\|x\|^2 \|z_J\|^2 + 2p\inner{x,z_J}^2 \pm O(p\epsilon) \mper\nonumber
	\end{align}
	\begin{align}
	\textup{RHS of}~\eqref{eqn:15} &  = 2\left(\inner{\Po(zz^{\top}), \Po(z_Jz_J^{\top})} - \inner{\Po(xx^{\top}), \Po(z_Jz_J^{\top})}\right)\tag{Since $M = zz^{\top}$}\\
	& = 2\|z_J\|^4 - 2\inner{x,z_J}^2  \pm O(p\epsilon) \mper\tag{by Theorem~\ref{thm:rank1_incohenrence_concentration} }
	\end{align}
	(Again notice that using $\tau$-relaxed second order \KKT condition does not effect the RHS by too much, so it does not change later steps.)
	Therefore plugging the estimates above back into equation~\eqref{eqn:15}, we have that 
	\begin{equation}
	p\|x\|^2 \|z_J\|^2 + 2p\inner{x,z_J}^2 \ge p\|z_J\|^4 \pm O(p\epsilon)\mcom \nonumber
	\end{equation}
		Using Cauchy-Schwarz, we have $\|x\|^2\|z_J\|^2\ge \inner{x,z_J}^2$, and therefore we obtain that $\|z_J\|^2\|x\|^2 \ge \frac{1}{3}\|z_J\|^4 - O(\epsilon)$. 
	
	Finally, we claim that $\|z_J\|^2 \ge 1/2$, which completes the proof since $\|x\|^2 \ge \frac{1}{3}\|z_J\|^2 - O(\eps) \ge 1/8$. 
	
	\begin{claim}\label{claim:z_L}
				Suppose $\alpha \ge \frac{4\mu}{\sqrt{d}}$ and $x$ satisfies $\|x\|_{\infty}\le 4\alpha$ and $\|x\|\le 2$. Let $J = \{i: |x_i|\le \alpha\}$. Then we have that $\|z_{J}\|\ge 1/2$. 	\end{claim}

	The claim can be simply proved as follows: Since $\|x\|^2 \le 2$ we have that $|J^c|\le 2/\alpha^2$ and therefore $\|z_{J^c}\|^2 \le 2\mu^2 /(d\alpha^2)$.  This further implies that $\|z_J\|^2 = \|z\|^2 - \|z_L\|^2 \ge (1-2\mu^2 /(d\alpha^2))  \ge \frac{1}{2}$ because $\alpha \ge \frac{2\mu}{\sqrt{d}}$.
					\end{proof}

\begin{lemma}[Lemma~\ref{lem:rank-1-reduction} restated]
	Suppose $x\in \mathcal{B}'$ satisfies that $\|x\|^2\ge 1/8$, under the same assumption as Lemma~\ref{lem:first-order-general-basic}. we have, 		\begin{equation}
	\Norm{\inner{x,z}z-\|x\|^2x} \le O(\epsilon)\nonumber
	\end{equation}
\end{lemma}

\begin{proof}
	Let $L = \{i: \|x_i\| \ge \alpha\}$. For $i\not\in L$, we have that $(\nabla R(x))_i = 0$. Therefore it suffices to prove that for every $i\in L$, 
	\begin{equation}
	(z_iz^{\top}x-x_i\|x\|)^2 \le (z_iz^{\top}x-x_i\|x\|- (\gamma\nabla R(x))_i)^2\nonumber
	\end{equation}
	It suffices to prov that 
	\begin{equation}
	(\nabla R(x))_i (x_i\|x\|^2-z_i\inner{z,x}) \ge 0\label{eqn:124}
	\end{equation}
	Since we have $\nabla R(x)_i = \gamma_i x_i$ for some $\gamma_i \ge 0$, we have
	\begin{align}
	(\nabla R(x))_i \cdot x_i\norm{x}^2 & = \inner{\gamma_i x_i, x_i\norm{x}^2} \nonumber\\
	& \ge \gamma_i x_i^2\norm{x}^2 \nonumber\\
	& \ge \frac{1}{\sqrt{8}}\gamma_i x_i^2\norm{x} \tag{since $\|x\|^2 \ge 1/8$}
	\end{align}
	On the other hand, we have 
	\begin{align}
	(\nabla R(x))_i\cdot z_i\inner{z,x} &= \gamma_i x_iz_i\inner{z,x}\nonumber\\
	& \le \frac{1}{4}\gamma_ix_i^2\|x\|\|z\|\tag{by $|x_i|\ge \alpha\ge 4|z_i|$ }
	\end{align}
	Therefore combining two equations above we obtain equation~\eqref{eqn:124} which completes the proof.  
\end{proof}

\section{Handling Noise}
\label{sec:noise}

Suppose instead of observing the matrix $ZZ^T$, we actually observe a noisy version $M = ZZ^T + N$, where $N$ is a Gaussian matrix with independent $N(0,\sigma^2)$ entries. In this case we should not hope to exactly recover $ZZ^T$ (as two close $Z$'s may generate the same observation). In this Section we show even with fairly large noise our arguments can still hold.

\begin{theorem}\label{thm:rankr-noisy}
Let $\hat{\mu} = \max\{ \mu, \sqrt{\frac{4\sigma d\sqrt{\log d}}{r}}\}$.
 Suppose $p \ge C\hat{\mu}^6\kappa^{12} r^4 d^{-1}\epsilon^{-2}\log^{1.5} d$ where $C$ is a large enough constant.  Let $\alpha = 2\hat{\mu} \kappa r/\sqrt{d}, \lambda \ge \hat{\mu}^2 rp/\alpha^2$. Then with high probability over the randomness of $\Omega$, any local minimum $X$ of $f(\cdot)$ satisfies 
 $$
 \|XX^T - ZZ^T\|_F \le \epsilon.
 $$

In fact, a noise level $\sigma\sqrt{\log d} \le \mu^2 r/d$ (when the noise is almost as large as the maximum possible entry) does not change the conclusions of Lemmas in this Section.
\end{theorem}

\begin{proof}
There are only three places in the proof where the noise will make a difference. These are: 1. The infinity norm bound of $M$, used in Lemma~\ref{lem:almost_incoherence_rankr}. 2. The LHS of first order \KKT condition (Equation~\eqref{eq:kkt1-rankr}). 3. The RHS of second order \KKT condition (Equation~\eqref{eq:kkt2-rankr}).

What we require in these three steps are: 1. $|M|_\infty$ should be smaller than $\mu^2r/d$. 2. $\inner{\Po(N), W}$ should be smaller than $|\inner{\Po(N),\Po(W)}| \le O(\sigma|Z|_\infty dr \log d + \sqrt{pd^2 r \sigma^2|W|_\infty \|W\|_F\log d})$. 3. $\|\Po(N)\| \le \epsilon p \|ZZ^T\|_F$. When we define the $\hat{\mu} = \max\{ \mu, \sqrt{\frac{4\sigma d\sqrt{\log d}}{r}}\}$, all of these are satisfied (by Lemma~\ref{lem:concentration-noise1} and \ref{lem:concentration-noise2}).

Now we can follow the proof and see $\delta \le c\epsilon\sminz/\kappa^2$ for small enough constant $c$, and By Lemma~\ref{prop:1} we know $\|XX^T-ZZ^T\|_F\le \epsilon$.
\end{proof}

\section{Finding the Exact Factorization}

In Section~\ref{sec:rankr}, we showed that any point that satisfies the first and second order necessary condition must satisfy $\|XX^T - ZZ^T\|_F \le c$ for a small enough constant $c$. In this section we will show that in fact $XX^T$ must be exactly equal to $ZZ^T$. The proof technique here is mostly based on the work of Sun and Luo\cite{sun2015guaranteed}. However we have to modify their proof because we use slightly different regularizers, and we work in the symmetric case. The main Lemma in \cite{sun2015guaranteed} can be rephrased as follows in our setting:

\begin{lemma} [Analog to Lemma 3.1 in \cite{sun2015guaranteed}] \label{lem:exact} Suppose $p\ge C\mu^4 r^6\kappa^4 d^{-1}\log d$ for large enough absolute constant $C$, and $\epsilon = \sigma_{min}(Z)^2/100$. with high probability over the randomness of $\Omega$, we have that for any point  $X$ in the set
	\begin{align}
	\mathcal{B}_{\epsilon} = \left\{X\in \R^{d\times r}: \|XX^T - ZZ^T\|_F \le \epsilon, \|X\|_{2\rightarrow \infty} \le \frac{16\mu \kappa r}{\sqrt{d}}\right\}\mcom
	\end{align}
	 		there exists a matrix $U$ such that $UU^T = ZZ^T$ and
$$
\inner{\nabla f(X), X - U} \ge \frac{p}{4}\|M-XX^T\|_F^2.
$$
As a consequence, any point $X$ in the set $\mathcal{B}$ that satisfies first order \KKT condition must be a global optimum (or, equivalently, satisfy $XX^T = ZZ^T$).
\end{lemma}

Recall $f(X)=\frac{1}{2}\|\Po(M-XX^T)\|_F^2+\lambda R(X)$. The proof of Lemma~\ref{lem:exact} consists of three steps:

1. The regularizer has nonnegative correlation with $(X-U)$: for any $U$ such that $UU^T = ZZ^T$, we have $\inner{\nabla R(X), X-U} \ge 0$. (Claim~\ref{claim:regularizer}). 

2. There exists a matrix $U$ such that $UU^T = ZZ^T$, and $U$ is close to $X$. (Claim~\ref{claim:close})

3. Argue that $\inner{\nabla f(x), X-U} \ge \frac{p}{4}\|\Po(M-XX^T)\|_F^2$ when $U$ is close to $X$. (See proof of Lemma~\ref{lem:exact}).

\noindent Before going into details, the first useful observation is that all matrices $U$ with  $UU^T = ZZ^T$ have the same row norm. 

\begin{claim}
\label{claim:Unorm}
Suppose $U,Z\in \R^{d\times r}$ satisfy  $UU^\top = ZZ^\top$. Then, for any $i\in [d]$ we have $\|U_i\| = \|Z_i\|$. Consequently, $\|U\|_F = \|Z\|_F$.
\end{claim}

\begin{proof}
Suppose $UU^\top = ZZ^{\top}$, then we have $U = ZR$ where $R$ is an orthonormal matrix. In particular, the $i$-th row of $U$ is equal to
$$
U_i = Z_i R.
$$
Since $\ell_2$ norm (and Frobenius norm) is preserved after multiplying with an orthonormal matrix, we know $\|U_i\| = \|Z_i\|$. The Frobenius norm bound follows immediately.
\end{proof}

Note that this simple observation is only true in the symmetric case. This Claims serves as the same role of the bounds on row norms of $U,V$ in the asymmetric case  (Propositions 4.1 and 4.2 of \cite{sun2015guaranteed}).

Next we are ready to argue that the regularizer is always positively correlated with $X-U$.

\begin{claim}
\label{claim:regularizer}
For any $U$ such that $UU^T = ZZ^T$, we have, 
$$\inner{\nabla R(X), X-U} \ge 0.
$$
\end{claim}

\begin{proof}
Since the regularizer is applied independently to individual rows, we can rewrite $\inner{\nabla R(X), X-U} = \sum_{i=1}^n \inner{\nabla R(X_i), X_i - U_i}$, and focus on $i$-th row.

For each row $X_i$, $\nabla R(X_i)$ is 0 when $\|X_i\| \le 2\mu\sqrt{r}/\sqrt{d}$. In that case $\inner{\nabla R(X_i), X_i - U_i} = 0$.

When $\|X_i\|$ is larger than $2\mu/\sqrt{d}$, we know $\nabla R(X_i)$ is always in the same direction as $X_i$. In this case $\lambda \nabla R(X_i) = \gamma X_i$ for some $\gamma > 0$ and $\|X_i\| \ge 2\mu\sqrt{r}/\sqrt{d} \ge 2\|Z_i\| = 2\|U_i\|$ (where last equality is by Claim~\ref{claim:Unorm}). Therefore by triangle inequality
$$
\inner{X_i, X_i - U_i} \ge \|X_i\|^2 - \|X_i\|\|U_i\| \ge \|X_i\|^2/2 > 0.
$$
This then implies $\inner{\lambda\nabla R(X_i), X_i - U_i} = \gamma \inner{X_i, X_i - U_i} > 0$.

\end{proof}

Next we will prove the gradient of $\frac{1}{2}\|\Po(M-XX^T)\|_F^2$ has a large correlation with $X-U$. This is analogous to Proposition 4.2 in \cite{sun2015guaranteed}. 

\begin{claim}
\label{claim:close}
Suppose $\|XX^T - M\|_F = \epsilon \le \sigma_{min}(Z)^2/100$, there exists a matrix $U$ such that $UU^T = M$ and $\|X-U\|_F \le 5\epsilon\sqrt{r}/\sigma_{min}(Z)^2$.
\end{claim}

\begin{proof}
Without loss of generality we assume $M $ is a diagonal matrix with first $r$ diagonal terms being $\sigma_1(Z)^2, \sigma_2(Z)^2,...,\sigma_r(Z)^2$ (this can be done by a change of basis). That is, we assume $M = \diag(\sigma_1(Z)^2,\dots, \sigma_r(Z)^2), 0,\dots,0)$. We use $M'$ to denote the first $r\times r$ principle submatrix of $M$.

We write $X = \begin{bmatrix}
V \\
W
\end{bmatrix}$ where $V$ contains the first $r$ rows of $X$, and $W\in \R^{(d-r)\times r}$ contains the remaining rows in $X$. We write similarly $U = \begin{bmatrix}
P \\
Q
\end{bmatrix}$ where $P$ and $Q$ denote the first $r$ rows and the rest of rows respectively.

In order to construct $U$, we first notice that $Q$ must be constructed as a zero matrix since $M$ has non-zero diagonal only on the top-left corner. A natural guess of $P$ then becomes a ``normalized'' version of $V$.

Concretely, we construct $P := VS = V(V^T(M')^{-1}V)^{-1/2}$ (where $S:=(V^T(M')^{-1}V)^{-1/2}$). Thus, the difference between $U$ and $X$ is equal to $\|U - X\|_F \le \|P - V\|_F + \|W\|_F$. 

Since $\|XX^T - M\|_F \le \epsilon$, we know $\|M' - VV^T\|_F^2 + 2\|VW^T\|_F^2 \le \epsilon^2$. In particular both terms are smaller than $\epsilon^2$. 

First, we bound $\|W\|_F$. Note that since $\|M' - VV^T\|_F \le \epsilon \le \sigma_{min}(Z)^2/100$, we know $\sigma_{min}(V)^2 \ge 0.99 \sigma_{min}(Z)^2$. Therefore $\sigma_{min}(V) \ge 0.9\sigma_{min}(Z)$. Now we know $\|W\|_F \le \|VW^T\|_F/\sigma_{min}(V) \le 2\epsilon/\sigma_{min}(Z)$. 

Next we bound $\|P-V\|_F^2$. Since $\|M' - VV^T\|_F \le \epsilon \le \sigma_{min}(Z)^2/100$, we know $(1 - 2\epsilon/\sigma_{min}(Z)^2) VV^T \preceq M' \preceq (1 + 2\epsilon^2/\sigma_{min}(Z)^2) VV^T$. This implies $\|V\|_F\le 1.1 \|Z\|_F$, and  $(1 - 2\epsilon/\sigma_{min}(Z)^2) I \preceq V^TM^{-1}V \preceq (1 + 2\epsilon/\sigma_{min}(Z)^2) I$. Therefore the matrix $S$ is also very close to identity, in particular, $\|S-I\| \le 2\epsilon/\sigma_{min}(Z)^2$. Now we know $\|P-V\|_F = \|V\|_F \|S-I\| \le 3\epsilon \|Z\|_F/\sigma_{min}(Z)^2$. Using the fact that $\|Z\|_F = 1$ we know $\|U - X\|_F \le \|P - V\|_F + \|W\|_F \le 5\epsilon\sqrt{r}/\sigma_{min}(Z)^2$.

\end{proof}

We can now combine this Claim with a sampling lemma to show $\|\Po((X-U)(X-U)^T)\|_F^2$ is small:

\begin{lemma}
\label{lem:upperboundb}
Under the same setting of Lemma~\ref{lem:exact}, with probability at least $1-1/(2n)^4$ over the choice of $\Omega$, if $U$ satisfies conclusion of Claim~\ref{claim:close}, then, $$
\|\Po((X-U)(X-U)^T)\|_F^2 \le \frac{p}{25}\|M-XX^T\|_F^2.
$$
\end{lemma}

Intuitively, this Lemma is true because $\|(X-U)(X-U)^T\|_F \le 25\|M-XX^T\|_F^2 r/\sigma_{min}(Z)^4$, which is much smaller than $\|M-XX^T\|_F$ when $\|M-XX^T\|_F$ is small. By concentration inequalities we expect $\|\Po((X-U)(X-U)^T)\|_F^2$ to be roughly equal to $p \|(X-U)(X-U)^T\|_F$, therefore it must be much smaller than $p\|M-XX^T\|_F^2$.
The proof of this Lemma is exactly the same as Proposition 4.3 in \cite{sun2015guaranteed} (in fact, it is directly implied by Proposition 4.3), so we omit the proof here. 
We also need a different concentration bound for the projection of the norm of the matrix $a = U(X-U)^T + (X-U)U^T$. Unlike the previous lemma, here we want $\|\Po(a)\|_F$ to be large.

\begin{lemma}
\label{lem:lowerbounda}
Under the same setting of Lemma~\ref{lem:exact}, let $a = U(X-U)^T + (X-U)U^T$ where $U$ is constructed as in Claim~\ref{claim:close}.  Then, with high probability, we have that for any $X\in \mathcal{B}_{\epsilon}$, 
$$
\|\Po(a)\|_F^2 \ge \frac{5p}{6} \|a\|_F^2.
$$
\end{lemma}

Intuitively this should be true because $a$ is in the tangent space $\{Z: Z = UW^T + (W') U^T\}$ which has rank $O(nr)$. The proof of this follows from Theorem 3.4 \cite{recht2011simpler}, and is written in detail in Equations (37) - (41) in \cite{sun2015guaranteed}.

Finally we are ready to prove the main lemma. The proof is the same as the outline given in Section 4.1 of \cite{sun2015guaranteed}. We give it here for completeness.

\begin{proof}[Proof of Lemma~\ref{lem:exact}]

Note that $f(X)$ is equal to $h(X)+\lambda R(X)$ where where $h(X) = \frac{1}{2}\|\Po(M-XX^T)\|_F^2$, and $R(X)$ is the regularizer.
 By Claim~\ref{claim:regularizer} we know $\inner{\nabla R(X), X-U} \ge 0$, so we only need to prove there exists a $U$ such that $UU^T = Z$ and $\inner{\nabla g(X), X-U} \ge \frac{p}{4}\|M-XX^T\|_F^2$.

Define $a = U(X-U)^T + (X-U)U^T$, $b = (U-X)(U-X)^T$, then $XX^T-M = a+b$ and $(X-U)X^T+X(X-U)^T = a + 2b$.

Now
\begin{align*}
\inner{\nabla h(X), X-U} & = 2\inner{\Po(XX^T - M)X, X-U} \\
& = \inner{\Po(XX^T - M), (X-U)X^T + X(X-U)^T} \\
& = \inner{\Po(a+b), \Po(a+2b)}\\
& = \|\Po(a)\|_F^2 + 2\|\Po(b)\|_F^2 + 3\inner{\Po(a), \Po(b)} \\
& \ge  \|\Po(a)\|_F^2+2\|\Po(b)\|_F^2-3\|\Po(a)\| \|\Po(b)\|.
\end{align*}

Let $\epsilon = \|M - XX^T\|_F$. Note that from Claim~\ref{claim:close} and Lemma~\ref{lem:upperboundb}, we know
$$
\|b\|_F\le \epsilon/10, \quad \|\Po(b)\|_F \le \sqrt{p}d/5.
$$
Therefore as long as we can show $\|\Po(a)\|_F$ is large we are done. This is true because $\|a\|_F \ge \|M-XX^T\|_F - \|b\|_F \ge 9\epsilon/10$. Hence by Lemma~\ref{lem:lowerbounda} we know

$$
\|\Po(A)\|_F^2 \ge \frac{5p}{6}\|a\|_F^2 \ge \frac{27}{40}p\epsilon^2.
$$

Combining the bounds for $\|\Po(a)\|_F$, $\|\Po(b)\|_F$, we know $\inner{\nabla g(X), X-U}  \ge \frac{p}{4}\|M-XX^T\|_F^2$. Together with the fact that $\inner{\nabla R(X), X-U} \ge 0$, we know 
$$
\inner{\nabla f(X), X - U} \ge \frac{p}{4}\|M-XX^T\|_F^2.
$$
\end{proof}

\section{Concentration inequality}\label{sec:concentration}

In this section we prove the concentration inequalities used in the main part. We first show that the inner-product of two low rank matrices is preserved after restricting to the observed entries. This is mostly used in arguments about the second order necessary conditions.

\begin{theorem}\label{thm:rank1_incohenrence_concentration}
		With high probability over the choice of $\Omega$, for any two rank-$r$ matrices $W,Z \in \R^{d\times d}$, we have 	\begin{equation}
	\left| \inner{\Po(W),\Po(Z)} - p \inner{W,Z} \right| \le O(|W|_\infty|Z|_\infty dr \log d + \sqrt{pdr|W|_{\infty}|Z|_{\infty}\Norm{W}_F\Norm{Z}_F\log d})\nonumber
		\end{equation}
	\end{theorem}

\begin{proof}
Since both LHS and RHS are bilinaer in both $W$ and $Z$, without loss of generality we assume the Frobenius norms of $W$ and $Z$ are all equal to 1. Note that in this case we should expect $|W|_\infty \ge 1/d$. 

Let $\delta_{i,j}$ be the indicator variable for $\Omega$, we know
$$
\inner{\Po(W, Z} = \sum_{i,j} \delta_{i,j} W_{i,j}Z_{i,j},
$$
and in expectation it is equal to $p \inner{W,Z}$. Let $Q = \sum_{i,j} (\delta_{i,j} - p)W_{i,j}Z_{i,j}$. We can then view $Q$ as a sum of independent entries (note that $\delta_{i,j} = \delta_{j,i}$, but we can simply merge the two terms and the variance is at most a factor 2 larger). The expectation $\E[Q] = 0$. Each entry in the sum is bounded by $|W|_\infty|Z|_\infty$, and the variance is bounded by
\begin{align*}
\Var[Q] & \le 
p \sum_{i,j} (W_{i,j} Z_{i,j})^2\\
& \le p \max_{i,j} |W_{i,j}|^2 \sum_{i,j} Z_{i,j}^2 \\
& \le p |W|_{\infty}^2.
\end{align*}

Similarly, we also know $\Var[Q] \le p |Z|_{\infty}^2$ and hence $\Var[Q] \le p |W|_{\infty}|Z|_{\infty}$.

Now we can apply Bernstein's inequality, with probability at most $\eta$,

$$
|Q-\E[Q]| \ge |W|_\infty|Z|_\infty \log 1/\eta + \sqrt{p |W|_{\infty}|Z|_{\infty} \log (1/\eta)}.
$$

By Proposition \ref{prop:epsnet}, there is a set $\Gamma$ of size $d^{O(dr)}$ such that for any rank $r$ matrix $X$, there is a matrix $\hat{X}\in \Gamma$ such that $\|X-\hat{X}\|_F \le 1/d^3$. When $W$ and $Z$ come from this set, we can set $\eta = d^{-Cdr}$ for a large enough constant $C$. By union bound, with high probability
$$
|Q-\E[Q]| \le O(|W|_\infty|Z|_\infty dr \log d + \sqrt{pdr |W|_{\infty}|Z|_{\infty}\log d}).
$$

When $W$ and $Z$ are not from this set $\Gamma$, let $\hat{W}$ and $\hat{Z}$ be the closest matrix in $\Gamma$, then we know $|\inner{\Po(W),\Po(Z)} - p \inner{W,Z} - (\inner{\Po(\hat{W}),\Po(\hat{Z})} - p \inner{\hat{W},\hat{Z}})| \le O(1/d^3) \ll |W|_\infty|Z|_\infty dr \log d$. Therefore we still have 
$$
|\inner{\Po(W),\Po(Z)} - p \inner{W,Z}| \le O(|W|_\infty|Z|_\infty dr \log d + \sqrt{pdr|W|_{\infty}|Z|_{\infty}\Norm{W}_F\Norm{Z}_F\log d}).
$$

\end{proof}

Next Theorem shows $\Po (XX^T)X$ is roughly equal to $pXX^TX$, this is one of the major terms in the gradient.

\begin{theorem}\label{thm:concentration_2}
When $p \ge \frac{C\nu^6 r\log^2 d}{d\epsilon^2}$ for a large enough constant $C$, 
With high probability over the randomness of $\Omega$, for any matrix $X \in \R^{d\times r}$ such that $\|X_i\| \le \nu\sqrt{\frac{1}{d}} \|X\|_F$, we have 	\begin{equation}
	\|\Po(XX^{\top})X - pXX^TX\|_F\le p\eps \|X\|_F^3
 	\end{equation}

 	\end{theorem}

\begin{proof}
	 	
Without loss of generality we assume $\|X\|_F = 1$. Let $\delta_{i,j}$ be the indicator variable for $\Omega$, we first prove 
the result when $\delta_{i,j}$ are independent, then we will use standard techniques to show the same argument works for $\delta_{i,j} = \delta_{j,i}$. 

Note that
$$
[\Po(XX^{\top})X]_i = \sum_{j} \delta_{i,j} \inner{X_i,X_j} X_j,
$$
whose expectation is equal to
$$
[pXX^TX]_i = p \sum_{j} \inner{X_i,X_j} X_j.
$$

We know $\|X_i\| \le \nu\sqrt{\frac{1}{d}}$, therefore each term is bounded by $\nu^3 (1/d)^{3/2}$. Let $Z_i$ be a random variable that is equal to $\|\Po(XX^{\top})X]_i - [pXX^TX]_i\|^2$, then it is easy to see $\E[Z_i] \le pd \nu^6 (r/d)^3 = p\nu^6/d^2$. and the variance $\Var[Z_i] = \E[Z_i^2] - \E[Z_i]^2 \le pd \nu^{12}(1/d)^6 + 2\E[Z_i]^2 \le 3\E[Z_i]^2$ (as long as $p > 1/d$). Our goal now is to prove $\sum_{i=1}^d Z_i \le p^2\epsilon^2$ for all $X$.

Let $\bar{Z}_i$ be a truncated version of $Z_i$. That is, $\bar{Z}_i = Z_i$ when $Z_i \le [2pd \nu^3 (1/d)^{3/2}]^2$, and $\bar{Z}_i = [2pd \nu^3 (1/d)^{3/2}]^2$ otherwise. It's not hard to see $\bar{Z}_i$ has smaller mean and variance compared to $Z_i$. Also, by vector's Bernstein's inequality (Lemma~\ref{lem:vecbernstein}), we know 
$$
\Pr[\sqrt{\bar{Z}_i} \ge  t] \le d\exp\left(-\frac{t^2}{\frac{3p\nu^6}{d^2} + 3t\frac{\nu^3}{d^{3/2}}}\right).
$$


Notice that this is only relevant when $t \le O(p\nu^3d^{-1/2})$ (because otherwise the probability is 0) and in that regime the variance term always dominates. Therefore $\bar{Z}_i$ is the square of a subgaussian random variable.

By the Bernstein's inequality, we know the moments of $\sqrt{\bar{Z}_i} - \E[\sqrt{\bar{Z}_i}]$ are dominated by a Gaussian distribution with variance $O(\E[\bar{Z}_i] \log d)$.

Now we can use the concentration bound for quadratics of the subgaussian random variables\cite{hsu2012tail}: we know that with probability $\exp(-t)$,
$$
\sum_{i=1}^d \bar{Z}_i \le O(\E[Z_i^2]\cdot(\log d)\cdot (d + 2\sqrt{dt} + 2t)).
$$

this means with probability $\exp(-Cdrlog d)$ with some large constant $C$, we know $\sum_{i=1}^d \bar{Z}_i \le O(p\nu^6r\log^2 d/d)$. The probability is low enough for us to union bound over all $X$ in a standard $\epsilon$-net such that every other $X$ is within distance $(\epsilon/d)^6$. Therefore we know with high probability for all $X$ in the $\epsilon$-net we have $\sum_{i=1}^d \bar{Z}_i \le O(p\nu^6r\log^2 d/d)$, which is smaller than $p^2 \epsilon^2$ when $p \ge \frac{C\nu^6 r\log^{1.5} d}{d\epsilon^2}$ for a large enough constant $C$.

For any $\hat{X}$ that is not in the $\epsilon$-net, let $X$ be the closest point of $X$ in the net, then $\|X-\hat{X}\|_F\le 1/d^6$, therefore the bound of $\hat{X}$ clearly follows from the bound of $X$. 

Now to convert sum of $\bar{Z}_i$ to sum of $Z_i$, notice that with high probability there are at most $2pd$ entries in $\Omega$ for every row. When that happens $Z_i$ is always bounded by $[2pd \nu^3 (1/d)^{3/2}]^2$ so $Z_i = \bar{Z}_i$. Let event $1$ be $\sum_{i=1}^d \bar{Z}_i \le p^2\epsilon^2$ for all $X$, and let event $2$ be that there are at most $2pd$ entries per row, we know with high probability both event happens, and in that case $\sum_{i=1}^d Z_i \le p^2\epsilon^2$ for all $X$.

\paragraph{Handling $\delta_{i,j} = \delta_{j,i}$} First notice that the diagonal entries $\delta_{i,i}$'s cannot change the Frobenius norm by more than $O(\nu^3(1/d)^{3/2}\cdot \sqrt{d}) \le p\epsilon$ so we can ignore the diagonal terms. Now for independent terms $\delta_{i,j}$, let $\gamma_{j,i} = \delta_{i,j}$, then by union bound both $\delta_{i,j}$ and $\gamma_{i,j}$ satisfy the equation, and by triangle's inequality $(\delta_{i,j}+\gamma_{i,j})/2$ also satisfies the inequality. Let $\tau_{i,j}$ be the true indicator of $\Omega$ (hence $\tau_{i,j} = \tau_{j,i}$), and $\tau'_{i,j}$ be an independent copy, we know $(\tau_{i,j}+\tau'_{i,j})/2$ has the same distribution as $(\delta_{i,j}+\gamma_{i,j})/2$ (for off-diagonal entries), therefore with high probability the equation is true for $(\tau_{i,j}+\tau'_{i,j})/2$. The Theorem then follows from the standard Claim below for decoupling (note that $\sup_{\|X\|_F = 1} \|\Po(XX^T)X-pXX^TX\|_F$ is a norm for the indicator variables of $\Omega$):

\begin{claim}
Let $X,Y$ be two iid random variables, then
$$
\Pr[\|X\| \ge t] \le 3\Pr[\|X+Y\| \ge \frac{2t}{3}].
$$
\end{claim}

\begin{proof}
Let $X,Y,Z$ be iid random variables then,
\begin{align*}
\Pr[X \ge t] & = \Pr[\|(X+Y)+(X+Z) - (Y+Z)\| \ge 2t]\\
& \le  \Pr[\|X+Y\| \ge 2t/3]+\Pr[\|X+Z\| \ge 2t/3]+\Pr[\|Y+Z\| \ge 2t/3]\\
& \le  3\Pr[\|X+Y\| \ge \frac{2t}{3}].
\end{align*}
\end{proof}

\end{proof}

Finally we argue that random sampling of a matrix gives a nice spectral approximation. This is a standard Lemma that is used in arguing about the $\Po(M)X$ term in the gradient ($\Po(M-XX^T)X$).

\begin{lemma}\label{lem:concentration-3}
	Suppose $W\in \R^{d\times d}$ satisfies that $|W|_{\infty}\le \frac{\nu}{d} \norm{W}_F$, then with high probability $(1-d^{-10})$ over the choice of $\Omega$,  
	\begin{equation}
	\|\Po(W) - pW\|\le \epsilon p \|W\|_F\mper\nonumber
	\end{equation}
	where $\epsilon = O(\nu \sqrt{\log d/(pd)})$. 
\end{lemma}

\begin{proof}
	Without loss of generality we assume $\|W\|_F = 1$. The proof follows simply from application of Bernstein inequality. We view $\Po(W)$ as 
	\begin{equation}
	\Po(W) = \sum_{i,j\in [d]^2} s_{ij}W_{ij}\delta_{ij}\nonumber
	\end{equation}
	where $\delta_{ij}\in \R^{d\times d}$ is the indicator matrix for entry $(i,j)$, and $s_{ij}\in \{0,1\}$ are independent Bernoulli variable with probability $p$ of being 1. 
	Then we have that $\E[\Po(W)] = pW$ and $\|s_{ij}W_{ij}\delta_{ij}\|\le \frac{\nu}{d}\|W\|_F$. Moreover, we compute the variance by 
	\begin{align}
	\sum_{i,j\in [d]^2}\E[s_{ij}W_{ij}^2\delta_{ij}^{\top}\delta_{ij}] & 	= \sum_{i,j\in [d]^2}\E[s_{ij}W_{ij}^2\delta_{jj}] \nonumber\\
	& = \sum_{j\in [d]}p\left(\sum_{i\in d}W_{ij}^2\right)\delta_{jj} 
	\end{align}
	Therefore 
	\begin{equation}
	\Norm{\sum_{i,j\in [d]^2}\E[s_{ij}W_{ij}^2\delta_{ij}^{\top}\delta_{ij}]}\le \frac{p\nu^2}{d}\nonumber
	\end{equation}
	Similarly we can control $	\Norm{\sum_{i,j\in [d]^2}\E[s_{ij}W_{ij}^2\delta_{ij}\delta_{ij}^{\top}]}$ by $p\nu^2/d$ (again notice that although $\delta_{i,j} = \delta_{j,i}$ the bounds here are correct up to constant factors). Then it follows from non-commutative Bernstein inequality~\cite{2010arXiv1004.3821I} that 
 	$$\Pr_{\Omega}\left[\|\Po(W)-p(W)\|\ge \epsilon p\right]\le d\exp(-2\epsilon^2pd/\nu^2)\mper$$
\end{proof}

\paragraph{Concentration Lemmas for Noise Matrix $N$}

Next we will state the concentration lemmas that are necessary when observed matrix is perturbed by Gaussian noise. The proof of these Lemmas are really exactly the same (in fact even simpler) than the corresponding Theorem that we have just proven. The first Lemma is used in the same settings as Theorem~\ref{thm:rank1_incohenrence_concentration}.

\begin{lemma} \label{lem:concentration-noise1}
Let $N$ be a random matrix with independent Gaussian entries $N(0,\sigma^2)$. With high probability over the support $\Omega$ and the Gaussian $N$, for any low rank matrix $W$, we have
$$
|\inner{\Po(N),\Po(W)}| \le O(\sigma|Z|_\infty dr \log d + \sqrt{pd^2 r \sigma^2|W|_\infty \|W\|_F\log d}
$$
\end{lemma}

\begin{proof}
The proof is exactly the same as Theorem~\ref{thm:rank1_incohenrence_concentration} as $|\inner{\Po(N),\Po(W)}|$ is a sum of independent entries that follows from the same Bernstein's inequality.
\end{proof}

Next we show that random sampling entries of a Gaussian matrix gives a matrix with low spectral norm.

\begin{lemma}\label{lem:concentration-noise2}
Let $N$ be a random Gaussian matrix with independent Gaussian entries $N(0,\sigma^2)$, with high probability over the choice of $\Omega$ and $N$, we have
$$
\|\Po(N)\| \le \epsilon p \sigma d,
$$
where $\epsilon = O(\sqrt{\log d/pd})$.
\end{lemma}

\begin{proof}
Again the proof follows from the same argument as Lemma~\ref{lem:concentration-3}.
\end{proof}

\section{Auxiliary Lemmas}
\label{sec:toolbox}

\begin{lemma}\label{lem:vecbernstein}[Bernstein inequality, c.f.~\cite{2010arXiv1004.3821I}]
	Let $v_i$'s be independent random vectors and $v = \sum_{i=1}^n v_i$. Suppose $\sigma^2 = \E[\sum_{i=1}^n \|v_i\|^2]$ and for all $i$ $\|v_i\| \le R$ with probability 1, then
	\[\Pr[\|v\| > t] \le d\exp(- t^2/(3\sigma^2 + 3tR)).\]
\end{lemma}

\begin{proposition}\label{prop:2-6norm}
	Let $a_1,\dots, a_r \ge 0$, $C\ge 0$. Then $C^4(a_1^2+\dots+a_r^2)\ge a_1^6+\dots+a_r^6$ implies that $a_1^2+\dots+a_r^2\le C^2r$ and that $\max a_i \le Cr^{1/6}$. 
\end{proposition}

\begin{proof}
	By Cauchy-Schwarz inequality, we have, 
	\begin{align}
	\Paren{\sum_{i=1}^r a_i^2} \Paren{	\sum_{i=1}^r a_i^6 } \ge \Paren{\sum_{i=1}^r a_i^4}^2 \ge \Paren{\frac{1}{r}\Paren{\sum_{i=1}^r a_i^2}^2}^2 \nonumber
	\end{align}
	Using the assumption and equation above we have that  $a_1^2+\dots+a_r^2\le C^2r$. This implies with the condition that $a_1^6+\dots+a_r^6\le C^6r$, which implis that $\max a_i \le Cr^{1/6}$.
\end{proof}

\begin{proposition} \label{prop:epsnet}
For any $\zeta \in (0,1)$, there is a set $\Gamma$ of rank $r$ $d\times d$ matrices, such that for any rank $r$ $d\times d$ matrix $X$ with Frobenius norm at most $1$, there is a matrix $\hat{X} \in \Gamma$ with $\|X-\hat{X}\|_F \le \zeta$. The size of $\Gamma$ is bounded by $(d/\zeta)^{O(dr)}$.
\end{proposition}

\begin{proof}
Standard construction of $\epsilon$-net shows that there is a set $P \subset \R^d$ of size $(d/\epsilon)^{O(d)}$ such that for any $\|u\|\le 1$, there is a $\hat{u} \in P$ such that $\|u-\hat{u}\| \le \epsilon$. Such construction can also be applied to matrices and Frobenius norm as that is the same as vectors and $\ell_2$ norm.

Here we let $\epsilon = 0.1\zeta$, and construct three sets $P_1,P_2,P_3$ where $P_1$ is an $\epsilon$-net for $d\times r$ matrices with Frobenius norm at most $\sqrt{r}$, $P_2$ is an $\epsilon$-net for $r\times r$ diagonal matrices whose Frobenius norm is bounded by $1$, and $P_3$ is an $\epsilon$-net for $r\times d$ matrices with Frobenius norm at most $\sqrt{r}$.

Now we define $\Gamma = \{\hat{U}\hat{D}\hat{V}|\hat{U}\in P_1, \hat{D}\in P_2, \hat{V}\in P_3\}$. Clearly the size of $\Gamma$ is as promised. For any rank $r$ $d\times d$ matrix $X$, suppose its Singular Value Decomposition is $UDV$, we can find $\hat{U} \in P_1$, $\hat{D}\in P_2$ and $\hat{V} \in P_3$ that are $\epsilon$ close to $U,D,V$ respectively. Therefore $\hat{U}\hat{D}\hat{V} \in \Gamma$ and it is easy to check
$$
\|UDV - \hat{U}\hat{D}\hat{V}\|_F \le 8\epsilon \le \zeta.
$$
\end{proof}

\end{document}